\documentclass[authoryear,preprint,review,12pt]{elsarticle}

\makeatletter
\def\ps@pprintTitle{%
  \let\@oddhead\@empty
  \let\@evenhead\@empty
  \def\@oddfoot{\reset@font\hfil\thepage\hfil}
  \let\@evenfoot\@oddfoot
}
\makeatother

\usepackage{amssymb}

\usepackage{amsmath}
\usepackage{amsthm}
\usepackage{booktabs}
\usepackage{algorithm}
\usepackage{algorithmic}
\usepackage{hyperref}
\usepackage{multirow}
\usepackage{amsfonts}
\usepackage{color}
\usepackage{graphicx}
\usepackage{float}
\usepackage{subfig}
\usepackage{lineno}
\usepackage{thmtools,xcolor}

\newtheorem{definition}{Definition} 
\newtheorem{assumption}{Assumption}

\newtheorem{lemma}{Lemma}

\newtheorem{theorem}{Theorem}
\newtheorem{remark}{Remark}

\let\oldalign\align
\let\oldendalign\endalign
\renewenvironment{align}
  {\linenomathNonumbers\oldalign}
  {\oldendalign\endlinenomath}

\let\oldequation\equation
\let\oldendequation\endequation
\renewenvironment{equation}
  {\linenomathNonumbers\oldequation}
  {\oldendequation\endlinenomath}

\let\oldmultline\multline
\let\oldendmultline\endmultline
\renewenvironment{multline}
  {\linenomathNonumbers\oldmultline}
  {\oldendmultline\endlinenomath}

\journal{}

\begin{document}

\begin{frontmatter}

\title{Multi-Agent Continuous Control with Generative Flow Networks}

\author[label1]{Shuang~Luo}
\ead{luoshuang@zju.edu.cn}
\author[label2]{Yinchuan~Li\corref{correspondingauthor}}
\ead{liyinchuan@huawei.com}
\author[label3]{Shunyu~Liu}
\ead{liushunyu@zju.edu.cn}
\author[label4]{Xu~Zhang}
\ead{connorbitzx@gmail.com}
\author[label2]{Yunfeng~Shao}
\ead{shaoyunfeng@huawei.com}
\author[label1]{Chao~Wu\corref{correspondingauthor}}
\ead{chao.wu@zju.edu.cn}

\cortext[correspondingauthor]{Corresponding author. This article has been accepted for publication by Neural Networks. The published version is available at \url{https://doi.org/10.1016/j.neunet.2024.106243}. Digital Obiect Identifer: 10.1016/j.neunet.2024.106243. \copyright 2024. This manuscript version is made available under the CC-BY-NC-ND 4.0 license \url{https://creativecommons.org/licenses/by-nc-nd/4.0/}}
\address[label1]{School of Public Affairs, Zhejiang University}
\address[label2]{Huawei Noah’s Ark Lab}
\address[label3]{College of Computer Science and Technology, Zhejiang University}
\address[label4]{School of Artificial Intelligence, Xidian University}

\begin{abstract}
  Generative Flow Networks (GFlowNets) aim to generate diverse trajectories from a distribution in which the final states of the trajectories are proportional to the reward, serving as a powerful alternative to reinforcement learning for exploratory control tasks. However, the individual-flow matching constraint in GFlowNets limits their applications for multi-agent systems, especially continuous joint-control problems. 
  In this paper, we propose a novel \emph{Multi-Agent generative Continuous Flow Networks}~(MACFN) method to enable multiple agents to perform cooperative exploration for various compositional continuous objects. 
  Technically, MACFN trains decentralized individual-flow-based policies in a centralized global-flow-based matching fashion. During centralized training, MACFN introduces a continuous flow decomposition network to deduce the flow contributions of each agent in the presence of only global rewards. Then agents can deliver actions solely based on their assigned local flow in a decentralized way, forming a joint policy distribution proportional to the rewards. To guarantee the expressiveness of continuous flow decomposition, we theoretically derive a consistency condition on the decomposition network.
  Experimental results demonstrate that the proposed method yields results superior to the state-of-the-art counterparts and better exploration capability. Our code is available at 
\url{https://github.com/isluoshuang/MACFN}.
\end{abstract}


\end{frontmatter}


\section{Introduction}\label{sec:introduction}
Generative Flow Networks~(GFlowNets) have recently been attracting increasing attention from research communities~\citep{bengio2021gflownet}, attributed to their capability to obtain various solutions for exploratory control tasks.
Unlike conventional Reinforcement Learning~(RL)~\citep{sutton2018reinforcement}, which aims to maximize the accumulative rewards for a single optimal sequence, GFlowNets are expected to generate a diverse set of high-return candidates with probabilities proportional to the given reward distribution~\citep{bengio2021flow}.
Specifically, the generation process of GFlowNets is to form a Directed Acyclic Graph~(DAG) consisting of the discrete nodes in a trajectory. Then GFlowNets compute the flow matching loss by traversing the inflows and outflows of each node. To realize continuous control, \citet{cfn} and \citet{lahlou2023theory} further extend GFlowNets with continuous flow matching and verify its effectiveness experimentally.

Despite the encouraging results, GFlowNets only consider the individual-flow matching constraint, restricting it to single-agent tasks. 
Many real-world environments inevitably involve cooperative multi-agent problems, such as robotics control~\citep{afrin2021resource}, traffic light control~\citep{wu2020multi, YANG2021265}, and smart grid control~\citep{haes2019survey}. 
Designing GFlowNets for Multi-Agent Systems~(MAS), where a group of agents works collaboratively for one common goal, is perceived as a significantly more challenging problem than the single-agent counterpart due to several peculiar limitations: 
\emph{1)~The curse of dimensionality}. A straightforward way to realize multi-agent GFlowNets is regarding the entire MAS as a single agent and optimizing a joint-flow matching loss. However, this way is often unacceptable as the joint continuous action-observation space grows exponentially with the number of agents. 
\emph{2)~Partial observability}. Learning independent flow networks for each agent can encounter non-stationarity due to inaccurate flow estimation without global information.
\emph{3)~Reward sparsity and multimodality}. The agents need to deduce their multimodal flow contributions given only the terminating reward.
A very recent work~\citep{magfn} tries to use multi-flow networks to mitigate the MAS problem, which, however, is still limited to the discrete space.

In this work, we thus propose \emph{Multi-Agent generative Continuous Flow Networks}, abbreviated as MACFN, to enhance GFlowNets for multi-agent continuous control tasks. 
Technically, we adopt a \emph{Centralized Training with Decentralized Execution}~(CTDE)~\citep{MADDPG} paradigm, where agents can learn decentralized individual-flow-based policies by optimizing the global-flow-based matching loss in a centralized manner.
During centralized training, MACFN constructs a continuous flow decomposition network, taking advantage of additional global information to disentangle the joint flow function into agent-wise flow functions. This centralized training mechanism can effectively mitigate the non-stationarity issue that arises from partial observability.
Furthermore, agents can make decisions relying solely on their local flow functions in a decentralized way, avoiding the
the curse of dimensionality that occurs as the number of agents increases.
To facilitate effective multi-agent flow decomposition, we establish theoretical consistency~on~the~decomposition network. 
This enables multiple agents to collaboratively generate diverse joint actions with probabilities proportional to the sparse reward signals received only upon termination.
Additionally, we introduce a sampling approach to approximate the integrals over inflows and outflows in continuous flow matching.
Our main contributions can be summarized as follows:
\begin{itemize}
\item To the best of our knowledge, this work is the first dedicated attempt towards extending GFlowNets to address the multi-agent continuous control problem. We introduce a novel method named MACFN, enabling multiple agents to generate diverse cooperative solutions using only terminating rewards.
\item We propose a continuous flow decomposition network that ensures consistency between global and individual flows. Consequently, agents can make decisions based solely on their respective flow contributions, while also benefiting from centralized training that incorporates additional global information for flow decomposition.
\item We use a sampling-based approach to approximate the flow distribution and theoretically analyze the convergence of the proposed sampling algorithm under the prediction error of the estimated parent nodes.
\item Experiments conducted on several multi-agent continuous control tasks with sparse rewards demonstrate that MACFN significantly enhances the exploration capability of the agent and practically outperforms the current state-of-the-art MARL algorithms.
\end{itemize}

\section{Related Work}\label{sec:related}
We briefly review recent advances closely related to this work, including generative flow networks and cooperative multi-agent reinforcement learning.

\subsection{Generative Flow Networks}
Generative Flow Networks (GFlowNets) aim to generate a diverse set of candidates in an active learning fashion, with the training objective that samples trajectories from a distribution proportional to their associated rewards.
In recent years, GFlowNets have attracted substantial attention in various applications, such as molecule discovery~\citep{bengio2021flow,pan2022generative,ekbote2022consistent}, Bayesian structure learning~\citep{deleu2022bayesian,nishikawa2022bayesian}, biological sequence design~\citep{malkin2022trajectory,madan2022learning,jain2022biological,zhang2022unifying}, and discrete images~\citep{zhang2022generative}.
Although these recent algorithms have achieved encouraging results, their development and theoretical foundations are constrained to environments with discrete spaces.
Naturally, there have been several efforts to extend GFlowNets for continuous structures.
\citet{cfn} define the flow of a continuous state as the integral of the complete trajectory passing through that state, extending the flow-matching conditions~\citep{bengio2021gflownet} to continuous domains.
\citet{lahlou2023theory} theoretically extend existent GFlowNet training objectives, such as flow-matching~\citep{bengio2021flow}, detailed
balance~\citep{bengio2021gflownet} and trajectory balance~\citep{malkin2022trajectory}, to spaces with discrete and continuous components.
Nevertheless, currently, GFlowNets cannot accommodate Multi-Agent Systems~(MAS)~\citep{qin2016recent,COMA}.
A very recent work~\citep{magfn} tries to tackle the MAS problem by using the multi-flow network, but is still limited to the discrete setting.

\subsection{Cooperative Multi-Agent Reinforcement Learning}
Cooperative Multi-Agent Reinforcement Learning (MARL) has emerged as a promising approach to enable autonomous agents to tackle various tasks such as autonomous driving~\citep{yu2019distributed,shalev2016safe}, video games~\citep{vinyals2019grandmaster,berner2019dota,kurach2020google} and sensor networks~\citep{zhang2011coordinated,ye2015multi}.
However, learning joint policies for multi-agent systems remains challenging. Training agents jointly~\citep{claus1998dynamics} means that agents select joint actions conditioned on the global state or joint observation, leading to computation complexity and communication constraints.
By contrast, training agents policy independently~\citep{tan1993multi}
tackles the above problem but suffers from non-stationarity.
A hybrid paradigm called Centralized Training with Decentralized Execution~(CTDE)~\citep{MADDPG,WANG2023359} combines the advantages of the above two methods and is widely applied in both policy-based~\citep{COMA} and value-based~\citep{QMIX} methods.
Policy-based methods~\citep{MADDPG,COMA,MAPPO,HATRPO} introduce a centralized critic to compute the gradient for the local actors.
Value-based methods~\citep{VDN,QMIX,QPLEX,peng2021facmac,ZHANG20221} decompose the joint value function into individual value functions to guide individual behaviors.
The goal of these methods is to maximize the expectation of accumulative rewards, where agents always sample action with~the~highest~return. 

To enhance the efficiency of multi-agent exploration, IRAT~\citep{wang2022individual} and LIGS~\citep{mguni2021ligs} propose using intrinsic rewards to motivate agents to explore unknown states. CMAE~\citep{liu2021cooperative} incentivizes collaborative exploration by having multiple agents strive to achieve a common goal based on unexplored states. PMIC~\citep{li2022pmic} facilitates better collaboration by maximizing mutual information between the global state and superior joint actions while minimizing the mutual information associated with inferior ones. However, the inherent complexity of multi-agent scenarios coupled with the potential inaccuracy of predicting reward and state can lead to instability. 
These exploration methods primarily rely on state uncertainty to guide agent learning, presenting agents with the significant challenge of identifying which states necessitate further exploration. Moreover, the state-uncertainty-based methods often benefit from local exploration, while our proposed MACFN focuses on the diversity of entire trajectories, leading to a long-term exploratory strategy. Additionally, the ultimate goal of these methods is to maximize the cumulative rewards for only a single optimal sequence. In contrast, MACFN aims to obtain a diverse set of high-return solutions, with the selection probabilities being proportional to the reward distribution, thereby promoting a more generalized exploration.

\section{Preliminaries}\label{sec:preliminaries}
In this section, we formally define the cooperative multi-agent problem under the Decentralized Partially Observable Markov Decision Process~(Dec-POMDP). Then we introduce the flow modeling of Generative Flow Networks~(GFlowNets).

\subsection{Dec-POMDP}
A fully cooperative multi-agent sequential decision-making problem can be modeled as a Dec-POMDP~\citep{Dec-POMDP}, which is formally defined by the tuple:
\begin{align}
M=<\mathcal{S}, \mathcal{A}, \mathcal{I}, T, R, \Omega, U >,
\end{align}
where $\mathcal{S}$ denotes the global state space, $\mathcal{A} = \mathcal{A}^1 \times \cdots \times \mathcal{A}^N$ denotes the joint action space for $N$ agents. Here both the state and action space are continuous.
We consider partially observable settings, and each agent $i \in \mathcal{I}$ can only access a partial observation $o^{i} \in {\Omega}$ according to the observation function $U(s,i)$. 
At each timestep $t$, each agent chooses an action $a^i_t \in \mathcal{A}^i$, forming a joint action $\boldsymbol{a}_t\in \mathcal{A}$, leading to a state transition to the next state $s_{t+1}$ in the environment according to the state transition function $s_{t+1} = T(s_t,\boldsymbol{a}_t)$. We assume that for any state pair $(s_t,s_{t+1})$, a certain $\boldsymbol{a}_t$ is the only way from $s_t$ to $s_{t+1}$.
A complete trajectory $\tau = (s_0,...,s_f)$ is defined as a sequence state of $\mathcal{S}$, where $s_0$ is the initial state and $s_f$ is the terminal state. 
$ {r_t} = R(s_t, \boldsymbol{a}_t): \mathcal{S} \times \mathcal{A} \rightarrow \mathbb{R}$ is the global reward function~shared~by~all~agents.

\subsection{GFlowNets}
GFlowNets consider the Dec-POMDP as a flow network~\citep{bengio2021flow}, which constructs the set of complete trajectories $\mathcal{T}$ as Directed Acyclic Graph~(DAG). Thus, the trajectory $\tau \in \mathcal{T}$ satisfies the acyclic constraint that $\forall s_j \in \tau, s_m \in \tau, j \neq m$, we get $s_j \neq s_m$.
Define $F(\tau)$ as a non-negative trajectory flow function. 
For a given state s, the state flow $F(s)=\sum_{\tau:s \in \tau} F(\tau)$ is defined as the total flow through the state. 
For a given state transition $s \rightarrow s^{\prime}$, the edge flow $F\left(s \rightarrow s^{\prime}\right)=\sum_{\tau: s \rightarrow s^{\prime} \in \tau} F(\tau)$ is define as the flow through the edge.
Define $P_F(\cdot)$ as the corresponding forward transition probability over the state~\citep{malkin2022trajectory}, i.e.,
\begin{align}
P_F(s_{t+1}|s_t) := \frac{F(s_t \rightarrow  s_{t+1})}{F(s_t)}.
\end{align}

GFlowNets aim to generate a distribution proportional to the given reward function. 
To achieve this goal, GFlowNets converge if they satisfy the flow-matching conditions~\citep{bengio2021flow}: for all states except the initial states, the flow incoming to a state must match the outgoing flow. For a continuous control task, the state flow $F(s)$ is calculated as the integral of all trajectory flows passing through the state $F(s) = \int_{\tau:s\in \tau} F(\tau) \mathrm{d} \tau$~\citep{cfn}.
For any state $s_{t}$, the continuous flow matching conditions are described as:
\begin{equation}
\int_{s_{t-1}\in \mathcal{P} (s_{t})}  F(s_{t-1} \rightarrow s_{t}) \mathrm{d} s_{t-1} =\\ \int_{s_{t+1} \in \mathcal{C} (s_{t})}  F(s_{t} \rightarrow s_{t+1}) \mathrm{d} s_{t+1} +R(s_{t}),
\end{equation}
where $\mathcal{P}(s_t)$ is the parent set of $s_t$, defined as $\mathcal{P}(s_t) = \{ s \in \mathcal{S} : T(s, a \in \mathcal{A}) = s_t \}$. Similarly, $\mathcal{C}(s_t)$ is the children set of $s_t$, defined as $\mathcal{C}(s_t) = \{ s \in \mathcal{S} : T(s_t, a \in \mathcal{A}) = s \} $.
The reward signals are sparse in our setting, where rewards are given only upon the termination of a trajectory and remain zero during all other times. 
Additionally, it is worth mentioning that for any given state that is neither initial nor terminal, the inflows are equal to the outflows. Terminal states serve as boundary conditions, i.e., $\mathcal{A}(s_f) = \emptyset$, where the inflows are equivalent to the aforementioned rewards, and no outflows are present. A transition $s \rightarrow s_f$ into the terminal state is defined as the terminating transition and the corresponding flow $F(s \rightarrow s_f)$ is defined as the terminating flow.

\section{Methodology}\label{sec:methodology}
In what follows, we first provide the theoretical formulation of \emph{Multi-Agent generative Continuous Flow Networks}~(MACFN). Then we further detail the training framework based on the proposed MACFN and summarize the complete optimization objective.

\subsection{MACFN: Theoretical Formulation}

In MACFN, each agent learns its individual-flow-based policies in a centralized global-flow-based matching manner. Define the joint edge flow as $F(s_t, \boldsymbol{a}_t) = F(s_t \rightarrow  s_{t+1})$, and the individual edge flow of agent $i$  as $F(o_t^i, a_t^i) = F(o_t^i \rightarrow  o_{t+1}^i)$, which only depends on each agent’s local observations. 
To enable efficient learning among agents, we propose Definition~\ref{def-decomposition} to learn an optimal flow decomposition from the joint flow, helping to deduce the flow contributions of each agent.

\begin{definition}[Global Flow Decomposition]
\label{def-decomposition}
For any state $s_{t}$ and $\boldsymbol{a}_t$, the joint edge flow $F(s_t, \boldsymbol{a}_t)$ is a product of individual edge flow $F(o_t, a_t)$ across $N$ agents, i.e.,
\begin{align}
F(s_t, \boldsymbol{a}_t) = \prod_{i=1}^N F_i(o_t^i, a_t^i).
\end{align}
\end{definition}

Based on Definition~\ref{def-decomposition}, we have Lemma~\ref{lm:decomposition} proved in \ref{PL1}.

\begin{lemma} 
\label{lm:decomposition} 
Let $\pi(\boldsymbol{a_t} \mid s_t)=\frac{F\left(s_t, \boldsymbol{a}_t\right)}{F\left(s_t\right)}$ denotes the joint policy, and $\pi_i\left(a_i \mid o_i\right)$ denotes the individual policy of agent $i$. 
Under Definition~\ref{def-decomposition}, we have
\begin{align}
\pi(\boldsymbol{a} \mid s) =\prod_{i=1}^N \pi_i\left(a_i \mid o_i\right).
\end{align}

\end{lemma}

\begin{remark}
    Lemma~\ref{lm:decomposition} provides consistency between individual and joint policies. It indicates that if the global flow decomposition satisfies Definition~\ref{def-decomposition}, the individual policies learned by each agent depend only on local observations. Since the joint policies are proportional to the given reward function, each agent conducts the individual policy solely proportional to that reward in the same way.
\end{remark}

We define the joint continuous outflows and the joint continuous inflows of state $s_t$ in Definitions~\ref{def-outflows} and \ref{def-inflows}.
\begin{definition}[Joint Continuous Outflows]\label{def-outflows}
For any state $s_{t}$, the outflows are the integral of flows passing through state $s_{t}$, i.e.,
\begin{equation}
    \int_{s\in \mathcal{C} (s_t)}  F(s_t \rightarrow s) \mathrm{d} s.
\end{equation}
\end{definition}

\begin{definition}[Joint Continuous Inflows]\label{def-inflows}
For any state $s_{t}$, its inflows are the integral of flows that can reach state $s_{t}$, i.e., 
\begin{equation}
    \int_{s\in \mathcal{P} (s_t)}  F(s\rightarrow s_t) \mathrm{d} s.
\end{equation}
\end{definition}

Since the flow function $F(o_t^i)$ of each agent is independent, we can further obtain Lemma~\ref{lm:inflowandoutflow} under Assumption~\ref{assumption0}. The proof is presented in \ref{PL2}.

\begin{assumption} \label{assumption0}
    Assume that for any state pair $(s_t,s_{t+1})$, there is a unique joint action $\boldsymbol{a_t}$ such that $T(s_t,\boldsymbol{a_t}) = s_{t+1}$, which means that taking action $\boldsymbol{a_t}$ in $s_t$ is the only way to get to $s_{t+1}$. And assume actions are the translation actions, i.e., $T(s, a) = s + a$.
\end{assumption}

Assumption~\ref{assumption0} is used in Lemma~\ref{lm:inflowandoutflow} to find the parent nodes when calculating the inflows. We define an inverse transition network $G_\phi$ to find $o_t^i$ when given $o_{t+1}^i$ and $a_t^i$, i.e., $o_{t}^i \leftarrow G_\phi(o_{t+1}^i, a_t^i)$. 
This assumption is a property of many environments with deterministic state transitions, such as robot locomotion and traffic simulation. In the robot locomotion scenario, each robot's movement can be considered a translation action. Here, the assumption holds as the action taken by the robots (joint action of all robots) uniquely determines their next positions (states). Similarly, in traffic simulation with multiple autonomous vehicles, each vehicle's movement from one position to another is also treated as a translation action.

\begin{lemma} 
\label{lm:inflowandoutflow} 
For any state $s_{t}$, its outflows and inflows are calculated as follows: 
\begin{align}
    &\int_{s\in \mathcal{C} (s_t)}  F(s_t \rightarrow s) \mathrm{d} s = \int_{\boldsymbol{a} \in \mathcal{A}} F(s_t,\boldsymbol{a}_t) \mathrm{d} \boldsymbol{a}_t 
    \\&= \int_{\boldsymbol{a} \in \mathcal{A}} \prod_{i=1}^N F_i(o_t^i, a_t^i) \mathrm{d} \boldsymbol{a}_t
    = \prod_{i=1}^N \int_{a_t^i \in \mathcal{A}^i} F_i(o_t^i, a_t^i) \mathrm{d} a_t^i,
\end{align}
and
\begin{align}
    &\int_{s\in \mathcal{P} (s_t)}  F(s\rightarrow s_t) \mathrm{d} s  = \int_{\boldsymbol{a}:T(s,\boldsymbol{a})= s_{t}} F(s,\boldsymbol{a}) \mathrm{d} \boldsymbol{a} \\
    &=\int_{\boldsymbol{a}:T(s,\boldsymbol{a})= s_{t}} \prod_{i=1}^N F_i(o^i, a^i) \mathrm{d} \boldsymbol{a} 
    =\prod_{i=1}^N\int_{a^i:T(o^i,a^i)= o_t^i} F_i(G_\phi(o_t^i,a^i), a^i) \mathrm{d} a^i,
\end{align}
where $\boldsymbol{a}$ is the unique action that transition to $s_t$ from $s$ and $o^i=G_\phi(o_t^i,a^i)$ with $o_t^i=T(o^i,a^i)$.
\end{lemma}

In order to parameterize the Markovian flows, we propose the joint continuous flow matching condition as Lemma~\ref{theorem1}.

\begin{lemma}[Joint Continuous Flow Matching Condition]
\label{theorem1}
Consider a non-negative function $\hat F(s_t,\boldsymbol{a}_t)$  taking a state $s_t \in \mathcal{S}$ and an action $\boldsymbol{a}_t\in \mathcal{A}$ as inputs. Then we have $\hat F$ corresponds to a flow if and only if the following continuous flow matching conditions are satisfied:
\begin{align}   
        \forall s_t > s_0, ~ \hat F(s_t) 
        =  \int_{s:T(s,\boldsymbol{a}) = s_t} \hat{F}(s,\boldsymbol{a}:s\rightarrow s_t) \mathrm{d} s = \prod_{i=1}^N \int_{a^i \in \mathcal{A}^i} ~ \hat F_i(o^i, a^i) \mathrm{d} a^i 
\end{align}
and
\begin{align}
        \forall s_t < s_f, ~ \hat F(s_t) =  \int_{\boldsymbol{a} \in \mathcal{A}} \hat F(s_t,\boldsymbol{a}_t) \mathrm{d} \boldsymbol{a}_t = \prod_{i=1}^N \int_{a_t^i \in \mathcal{A}^i} ~ \hat F_i(o_t^i, a_t^i) \mathrm{d} a_t^i.
\end{align}
Furthermore, $\hat F$ uniquely defines a Markovian flow $F$ matching $\hat F$ such that
\begin{equation}
    F(\tau) = \frac{\prod_{t=1}^{n+1} \hat F(s_{t-1} \rightarrow s_t)}{\prod_{t=1}^{n}\hat F(s_t)}.
\end{equation}
\end{lemma}

\begin{proof}
The proof is trivial by following the proof of Theorem $1$ in~\citet{cfn}.
\end{proof}

\begin{remark}
    Lemma~\ref{theorem1} provides the flow matching condition for joint continuous flow. ~\citet{bengio2021gflownet} initially introduces the flow matching condition, demonstrating that non-negative functions of states and edges can be used to define a Markovian flow. Then ~\citet{cfn} formalizes this condition for continuous state and action spaces. In multi-agent scenarios, we extend the flow matching condition to Lemma~\ref{theorem1}, which guarantees the existence of a distinct Markovian flow when the non-negative flow function satisfies the joint flow matching conditions.
\end{remark}

Based on Lemma~\ref{theorem1}, we propose the following joint continuous loss function to train the flow network:

\begin{equation}
    \mathcal{L}(\tau)=\sum_{s_t \in \tau \backslash\{s_0\}}\left[\prod_{i=1}^N \int_{a^i \in \mathcal{A}^i} F_i(o^i, a^i) \mathrm{d} a^i - R(s_t) - \prod_{i=1}^N \int_{a_t^i \in \mathcal{A}^i} F_i(o_t^i, a_t^i) \mathrm{d} a_t^i\right]^2.
\end{equation}

Since the inflows and outflows of states in the continuous control task cannot be calculated directly to complete the flow matching condition, we solve the integral with a sampling-based approach. Specifically, we sample $K$ flows independently and uniformly from the continuous action space $\mathcal{A} = \mathcal{A}^1 \times \cdots \times \mathcal{A}^N$ for approximate estimation and match the sampled flows.
Then, we present Lemma~\ref{lm:expectation} proved in \ref{PL5}, which shows that the expectation of sampled inflow and outflow is the true inflow and outflow.

\begin{lemma} 
\label{lm:expectation}
Let $\{a^{1,k},...,a^{N,k}\}_{k=1}^{K}$ be sampled independently and uniformly from the continuous action space $\mathcal{A}^1 \times \cdots \times \mathcal{A}^N$.
Assume $G_{\phi^\star}$ can optimally output the actual state $o_t^i$ with $(o_{t+1}^i,a_t^i)$. Then for any state $s_t \in \mathcal{S}$, we have
\begin{equation}
    \mathbb{E}\left[\frac{\mu(\mathcal{A})}{K} \sum_{k=1}^K \prod_{i=1}^N  F(o_t^{i},a_t^{i,k}) \right]
    = \int_{\boldsymbol{a} \in \mathcal{A}}F(s_t,\boldsymbol{a})\mathrm{d} \boldsymbol{a}
\end{equation}
and
\begin{equation}
    \mathbb{E}\bigg[\frac{\mu(\mathcal{A})}{K} \sum_{k=1}^K \prod_{i=1}^N F(G_{\phi^\star} (o_t^{i}, a_t^{i,k}),a_t^{i,k})\bigg]
    =  \int_{\boldsymbol{a}:T(s,\boldsymbol{a})=s_t}F(s,\boldsymbol{a}) \mathrm{d} \boldsymbol{a}.
\end{equation}
\end{lemma}

It is worth noting that, unlike the single-agent task, if we estimate the parent node through network $G_{\phi}$, an accurate $s_{t}$ can be estimated through $s_{t+1}$ and $\boldsymbol{a}_{t+1}$ (i.e., $s_{t} = G_{\phi}(s_{t+1},\boldsymbol{a}_{t+1})$) under the assumption that for any state transition pair $(s_t,s_{t+1})$, there is a unique action $\boldsymbol{a}_t$ such that $T(s_t,\boldsymbol{a}_t) = s_{t+1}$~\citep{cfn}. However, an accurate $o_{t}^i$ cannot be directly estimated through $o_{t+1}^i$ and $a_{t+1}^{i}$, because the parent node is related to all $\{o_{t+1}^i\}_{i=1}^N$, which cannot be obtained only through partial observations, i.e., $o_{t}^i \neq G_{\phi}(o_{t+1}^i, a_{t+1}^i)$. Hence, we analyze the convergence of the proposed sampling algorithm under the condition that the estimated parent~nodes~have~errors. 

In Assumption~\ref{parent-error}, we assume that the estimation error of the parent node will introduce an error to the flow function, which is related to the number of agents $N$ and the number of sampled actions $K$. Obviously, the larger the number of agents, the smaller the proportion of $o_i$ observing the global state, so the larger the error. In addition, the larger the sample size of actions, the smaller the error.

\begin{assumption}\label{parent-error} Let $G_{\phi^\star}$ be a function that optimally outputs the actual state $s_t$ with $(s_{t+1},a_t)$ and $G_{\phi}$ be the neural network that outputs the estimation of $o_t$ based on $(o_{t+1},a_t)$. There exist parameters $\alpha,\beta > 0$ such that the neural network $G_{\phi}$ satisfies

\begin{equation}
    \left|\prod_{i=1}^N F(G_{\phi} (o_{t+1}, a_t^i),a_t^i) -  F(G_{\phi^\star} (s_{t+1}, \boldsymbol{a}_t),\boldsymbol{a}_t) \right | \le \frac{N^\alpha}{K^\beta}.
\end{equation}
\end{assumption}

Based on Assumption~\ref{parent-error}, we present the following Theorem~\ref{thm: tail} proved in \ref{PT2}.

\begin{theorem} \label{thm: tail}
Let $\{a^{1,k},...,a^{N,k}\}_{k=1}^{K}$ be sampled independently and uniformly from the continuous action space $\mathcal{A}^1 \times \cdots \times \mathcal{A}^N$. Assume $G_{\phi^\star}$ can optimally output the actual state $o_t^i$ with $(o_{t+1}^i,a_t^i)$. For any bounded continuous action  $\boldsymbol{a} \in \mathcal{A}$, any state $s_t \in \mathcal{S}$ and any $\delta>0$, we have~\begin{multline} 
    \mathbb{P}\left(\Big{|}\frac{\mu(\mathcal{A})}{K} \sum_{k=1}^K \prod_{i=1}^N  F(o_t^{i},a_t^{i,k}) -\int_{\boldsymbol{a} \in \mathcal{A}}F(s_t,\boldsymbol{a})\mathrm{d} \boldsymbol{a} \Big{|}\right.  \\  \left. \ge \frac{\delta L\mu(\mathcal{A}) \rm{diam}(\mathcal{A}) }{\sqrt{K}} \right)
    \le 2\exp\left(-\frac{\delta^2}{2}\right)
\end{multline}
and 
\begin{multline}
   \mathbb{P}\left( \Big{|}\frac{\mu(\mathcal{A})}{K} \sum_{k=1}^K \prod_{i=1}^N F(G_{\phi} (o_t^{i}, a_t^{i,k}),a_t^{i,k})    - \int_{\boldsymbol{a}:T(s,\boldsymbol{a})=s_t}F(s,a) \mathrm{d} \boldsymbol{a} \Big{|} \right.  \\  \left.
     \ge \frac{\delta L\mu(\mathcal{A}) [\rm{diam}(\mathcal{A}) + \rm{diam}(\mathcal{S})]}{\sqrt{K}} + \frac{\mu(\mathcal{A}) N^\alpha}{K^\beta} \right)\\ 
   \le 2\exp\left(-\frac{\delta^2}{2}\right),
\end{multline}
where $L$ is the Lipschitz constant of the function $F(s_t,a)$, $\rm{diam}(\mathcal{A})$ denotes the diameter of the action space $\mathcal{A}$ and $\mu(\mathcal{A})$ denotes the measure of the action space $\mathcal{A}$.
\end{theorem}

\begin{remark}
    Theorem~\ref{thm: tail} establishes the error bound between the sampled flows and the actual flows based on multi-agent flow sampling and parent node approximation, extending Theorem 2 in \citet{cfn} to the multi-agent case.
    The error decreases in the tail form as $K$ increases, which indicates that the sampling flows serve as an accurate way to estimate the actual flows. 
    As long as the number of samples $K$ is sufficient, there will be a minor error between our sampled flow and the actual flow with a high probability, ensuring that our algorithm can converge well in practice.
\end{remark}

\subsection{MACFN: Training Framework}
As shown in Figure~\ref{framework}, our framework adopts the Centralized Training with Decentralized Execution~(CTDE) paradigm. Each agent learns its individual-flow-based policy by optimizing the global-flow-based matching loss of the continuous flow decomposition network. During the execution, the continuous flow decomposition network is removed, and each agent acts according to its local policy derived from its flow function. Our method can be described by three steps: 1) flow sampling; 2) flow decomposition; 3) flow matching.

\textbf{Flow Sampling.}
Since the action space $\mathcal{A} = \mathcal{A}^1 \times \cdots \times \mathcal{A}^N$ are continuous and independent, we use a Monte Carlo integration approach and sample actions from $\mathcal{A}^i$ for each agent $i$ uniformly and independently. The individual flow function $F$ is parameterized by $\theta$. Then we calculate the corresponding edge flow $\{F_\theta(o^i,a^{i})\}$ for flow estimation to approximate the flow distributions. Then, we normalized the flow distributions by the softmax function to obtain action probability distributions. Each agent selects an action according to the probability distributions, forming the joint action~$\boldsymbol{a}$. Naturally, actions with higher flows are more likely to be sampled. We repeat this sampling process until the agents reach the final state, obtaining the complete trajectories. In this way, we approximately sample actions based on their corresponding probabilities.

\begin{figure*}[!t]
    \centering
    \setlength{\abovecaptionskip}{-0.0cm} 
  \setlength{\belowcaptionskip}{-0.0cm} 
    \includegraphics[width=1.0\textwidth]{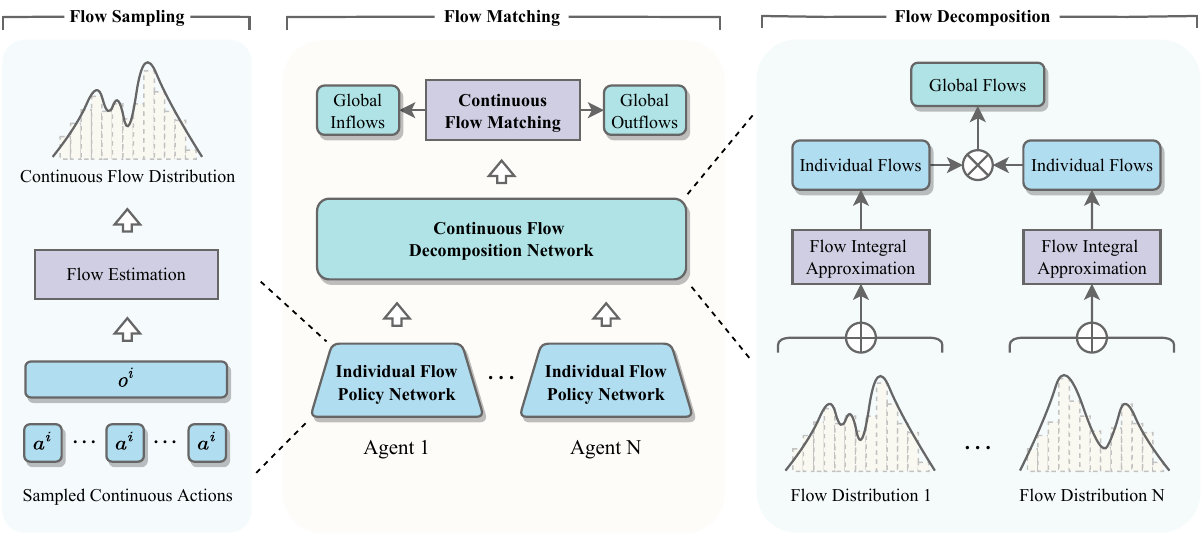}
    \caption{The overall framework of MACFN. \textbf{Left:}
    We uniformly sample continuous actions for each agent and conduct flow estimation to obtain the individual flow distribution. Each agent can select actions according to its own flow distribution. \textbf{Right:} To realize flow decomposition,
    we first summarize the sampled flows of each agent to approximate the integral of individual flow distribution. Then we multiply the flow of each agent to obtain the global inflows~(outflows). \textbf{Middle:} We update the overall flow networks by the continuous flow matching loss, i.e., inflows equal to outflows.}
    \label{framework}
  \end{figure*}

\textbf{Flow Decomposition.}
Then, we introduce the flow decomposition network to deduce the flow contributions of each agent.
Specifically, the flow decomposition network enables individual flow networks to learn by backpropagating gradients in the presence of only global rewards rather than from any agent-specific rewards.
To calculate the inflows and outflows of the sampled states, we independently and uniformly sample~$\hat{K}$ actions $\{a^{i,\hat{k}}\}_{\hat{k}=1}^{\hat{K}}$ from $\mathcal{A}^i$ for each agent $i$. Since the individual flow networks of agents are independent and we sample action independently, we have $\hat{K}=\log_N K$, where $N$ is the number of agents.
Then we calculate the individual flow~$F_i(o_t^i,a^{i,\hat{k}}), \hat{k} = 1,...,\hat{K}$. We sum the sampled $\hat{K}$ flows for each agent $i$ to obtain an approximation of the integral of the outflows (resp. inflows) of each agent. Then, we multiply the corresponding flows of~$N$ agents to obtain an approximation of the outflows (resp. inflows) of the sampled states.
The inverse transition network $G_\phi$ is used to find the parent sets of agents when calculating the inflows.

\textbf{Flow Matching.}
In order to update the individual flow network of each agent, we calculate the continuous flow-matching loss~\citep{bengio2021flow} based on the approximation of global state outflows (resp. inflows), i.e., for any states $s_t$ except for initial, the flow incoming to $s_t$ must equal to the outgoing flow. For sparse reward environments, $R(s_t)=0$ if $s_t \neq s_f$.

Since the magnitude of the flows at the root nodes and leaf nodes in the trajectories may be inconsistent, we adopt log-scale operation~\citep{bengio2021flow} to obtain the following loss function:

\begin{equation}\label{approximate-loss-log}
    \tilde{\mathcal{L}}(\tau, \epsilon ; \theta)=\sum_{s_t \in \tau \neq s_0}[\log (\epsilon+\text{Inflows})-\log (\epsilon+\text{Outflows})]^2,
\end{equation}

\noindent where
\begin{align}
\text{Inflows}\label{inflow-loss} & := \prod_{i=1}^N\left(\sum_{\hat{k}=1}^{\hat{K}} \exp F_{\theta_i}^{\log}(G_{\phi} (o_t^{i}, a_t^{i,\hat{k}}),a_t^{i,\hat{k}})\right), \\
\text{Outflows}\label{outflow-loss} & :=R\left(s_t\right)+\prod_{i=1}^N\left(\sum_{\hat{k}=1}^{\hat{K}} \exp F_{\theta_i}^{\log}(o_t^i,a_t^{i,\hat{k}})\right).
\end{align}

To this end, each agent updates its own flow policy by backward propagation gradients through the flow decomposition network, enabling learning in a centralized manner and executing in a decentralized way. In this way, each agent can solely select actions based on individual flow policy, where the joint flow is proportional to the given reward function when interacting with the environment. To make the proposed method clearer for readers, we provide the pseudocode of MACFN in Algorithm~\ref{pseudocode}.

\begin{algorithm}[!t]\footnotesize
	\caption{MACFN}
        \textbf{Initialize:} individual flow network $F_{\theta}$; inverse transition network $G_\phi$; replay buffer $\mathcal{D}$; sampled flow number $\hat{K}$
	\begin{algorithmic}[1]\label{pseudocode}
	\REPEAT
        \STATE Sample an episode according to flow distribution
        \STATE Store the episode $\{(s_t,\boldsymbol{a}_t,r_t,s_{t+1})\}_{t = 0}^{T}$ to $\mathcal{D}$ 
        \STATE Sample a random batch of episodes from $\mathcal{D}$ 
        \STATE Train inverse transition network $G_\phi$
	\FOR{each agent $i$}
        \STATE Uniformly sample $\hat{K}$ actions $\{a^{i,\hat{k}} \in \mathcal{A}^i\}_{\hat{k} = 1}^{\hat{K}}$ 
        \STATE Compute parent nodes $\{G_{\phi}(o_i,a^{i,\hat{k}})\}_{\hat{k}=1}^{\hat{K}}$ 
        \ENDFOR
        \STATE Compute inflows according to Eq.~(\ref{inflow-loss}):
        \STATE \quad\;$\prod_{i=1}^N\left(\sum_{\hat{k}=1}^{\hat{K}} \exp F_{\theta_i}^{\log}(G_{\phi^\star} (o_t^{i}, a_t^{i,{\hat{k}}}),a_t^{i,{\hat{k}}})\right)$
        \STATE Compute outflows according to Eq.~(\ref{outflow-loss}):
        \STATE \quad\;$R\left(s_t\right)+\prod_{i=1}^N\left(\sum_{\hat{k}=1}^{\hat{K}} \exp F_{\theta_i}^{\log}(o_t^i,a_t^{i,{\hat{k}}})\right)$
        \STATE Update flow network $F_{\theta}$ according to Eq.~(\ref{approximate-loss-log})
	\UNTIL converge
	\end{algorithmic}
	\end{algorithm}

\section{Experiments}\label{sec:experiments}
To demonstrate the effectiveness of the proposed MACFN, we conduct experiments on several multi-agent cooperative environments with sparse rewards, including the Multi-Agent Particle Environment~(MPE)~\citep{MADDPG,mordatch2018emergence} and Multi-Agent MuJoCo~(MAMuJoCo) \citep{peng2021facmac}. In this section, we first provide the details for the environments. Then the compared methods and parameter settings are introduced. The comparison results of different baselines are reported to evaluate the performance of MACFN.

\subsection{Environments}

First, we evaluate the proposed MACFN on various MPE scenarios with sparse reward, including \emph{Robot-Navigation-Sparse}, \emph{Food-Collection-Sparse}, and \emph{Predator-Prey-Sparse}. The visualization of these scenarios is shown in Figure~\ref{vis_env}. We consider the cooperative setting where a group of agents collaboratively works towards achieving a common goal, and only sparse rewards~are~provided.

\begin{figure}[!t]
	\centering
	\includegraphics[width=1.0\textwidth]{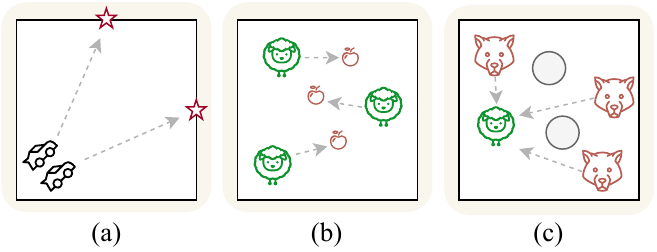}
 	\caption{Visualization of different MPE scenarios, including (a)~Robot-Navigation-Sparse, (b)~Food-Collection-Sparse, and (c)~Predator-Prey-Sparse.} 
	\label{vis_env}
\end{figure}

\emph{Robot-Navigation-Sparse} is a continuous navigation task with a maximum episode length of $12$. This scenario consists of $N$ agents, placed at the starting point of the environment. At each episode, the agents must navigate to different $N$ destinations and obtain rewards based on the minimum distance between the destinations and the agents.
In \emph{Food-Collection-Sparse}, $N$ agents need to occupy $N$ food locations cooperatively. Within the partially observed horizon, agents can observe the relative positions of other agents and food locations. The more food agents occupy, the greater the reward they receive, so agents need to infer the landmarks they must cover to avoid multiple agents grabbing a food location together.  The reward is calculated based on how far any agent is from each food.
\emph{Predator-Prey-Sparse} consists of $N$ cooperating predator agents, $2$ large landmarks, and $1$ prey around a randomly generated position. Predator agents move toward the prey while avoiding the landmark impeding the way. When the episode ends, the predator agents are rewarded based on their distance from the target prey.
The maximum episode length for both \emph{Food-Collection-Sparse} and \emph{Predator-Prey-Sparse} is $25$.

To further verify the effectiveness of the proposed MACFN, we compare our method with different baselines on a more challenging benchmark, named MAMuJoCo~\citep{peng2021facmac}. MAMuJoCo is a benchmark based on OpenAI's Mujoco Gym environments for continuous multi-agent robotic control, where multiple agents within a single robot have to solve a task cooperatively. We conduct experiments on several MAMuJoCo scenarios with sparse rewards, including \emph{2-Agent-Reacher-Sparse}, \emph{2-Agent-Swimmer-Sparse}, and \emph{3-Agent-Hopper-Sparse}. The maximum episode length for these scenarios is set to $50$.

In \emph{2-Agent-Reacher-Sparse}, Reacher is a two-jointed robotic arm, each joint controlled by an individual agent. The goal of the agents is to move the robot’s fingertip toward a random target. At the end of each episode, the reward is obtained by how far the fingertip of the Reacher is from the target, with greater rewards granted for shorter distances.
In \emph{2-Agent-Swimmer-Sparse}, Swimmer is a three-segment entity with two articulation rotors, each rotor controlled by an individual agent. The agents aim to propel it rightward quickly by applying torque to the rotors.
In \emph{3-Agent-Hopper-Sparse}, Hopper consists of four body parts connected by three hinges, each hinge controlled by an individual agent. The objective here is to achieve rightward hops by manipulating the hinges' torques.
For \emph{2-Agent-Swimmer-Sparse} and \emph{3-Agent-Hopper-Sparse}, the reward increases with the distance traveled to the right from the start point.

\subsection{Settings}

We compare MACFN with various baseline methods such as Independent DDPG~(IDDPG)~\citep{DDPG}, MADDPG~\citep{MADDPG}, COVDN~\citep{VDN}, COMIX~\citep{QMIX}, and FACMAC~\citep{peng2021facmac},  on these environments and show consistently large superior to the state-of-the-art counterparts, as well as better exploration capability. 
We run all experiments using the Python MARL framework (PyMARL)~\citep{SMAC,peng2021facmac}.
The hyperparameters of all methods are consistent, following the original settings to ensure comparability.
For all experiments, the optimization is conducted using Adam with a learning rate of $3 \times 10^{-4}$. A total of $1$M timesteps is used in the MPE scenarios, while a total of $2$M timesteps is used in the MAMuJoCo scenarios. All experimental results are illustrated with the mean and the standard deviation of the performance over five random seeds for a fair comparison. We adopt the average test return and the number of distinctive trajectories as performance metrics. 
For comparing the average test return, the proposed MACFN samples actions based on the maximum flow output. For the comparison of distinctive trajectories, it samples actions according to the flow distribution.
During the training process, we collect a total of $10,000$ trajectories.
The number of distinctive trajectories is calculated based on the distance between different generated valid trajectories, where we discard the particularly similar one using a threshold.

\subsection{Results}

The average test return of different MPE scenarios is shown in Figure~\ref{fig_reward}(a)-(c). For comparison, MACFN greedily selects the action with the largest edge flow. Our proposed MACFN demonstrates encouraging results in both scenarios.
In the easy \emph{Robot-Navigation-Sparse} scenario, two agents navigate from a fixed starting point to two fixed destinations, requiring relatively little collaborative exploration. Thus, several baselines, including MADDPG and COVDN, can also achieve promising results only considering the highest rewarding trajectories, while the exploratory benefit brought by MACFN is not obvious. 
However, in the more challenging scenarios, i.e., \emph{Food-Collection-Sparse} and \emph{Predator-Prey-Sparse} environments, our approach consistently outperforms baselines by a large margin as the difficulty increases.
In addition, it is noticeable that, compared with the baselines, the performance of MACFN is more stable across training. All environments require the agents to collaborate and explore to achieve their goals, especially in the sparse reward setting. 
While the baselines are laborious to explore in challenging \emph{Food-Collection-Sparse} and \emph{Predator-Prey-Sparse} environments, our approach can efficiently search for valuable trajectories with stable updates. It implies that MACFN successfully learns flow decomposition to deduce different flow contributions and coordinate different agents, so the agents can discover diverse cooperation patterns to achieve their goals.

\begin{figure*}[!t]
	\centering
 \subfloat{\quad\;\includegraphics[width=1.0\textwidth]{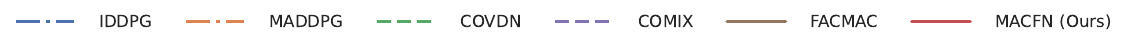}}
 \vspace{-0.5em}
 \\    
    \addtocounter{subfigure}{-1}
    \subfloat[][Robot-Navigation-Sparse]{
	\includegraphics[width=0.33\textwidth]{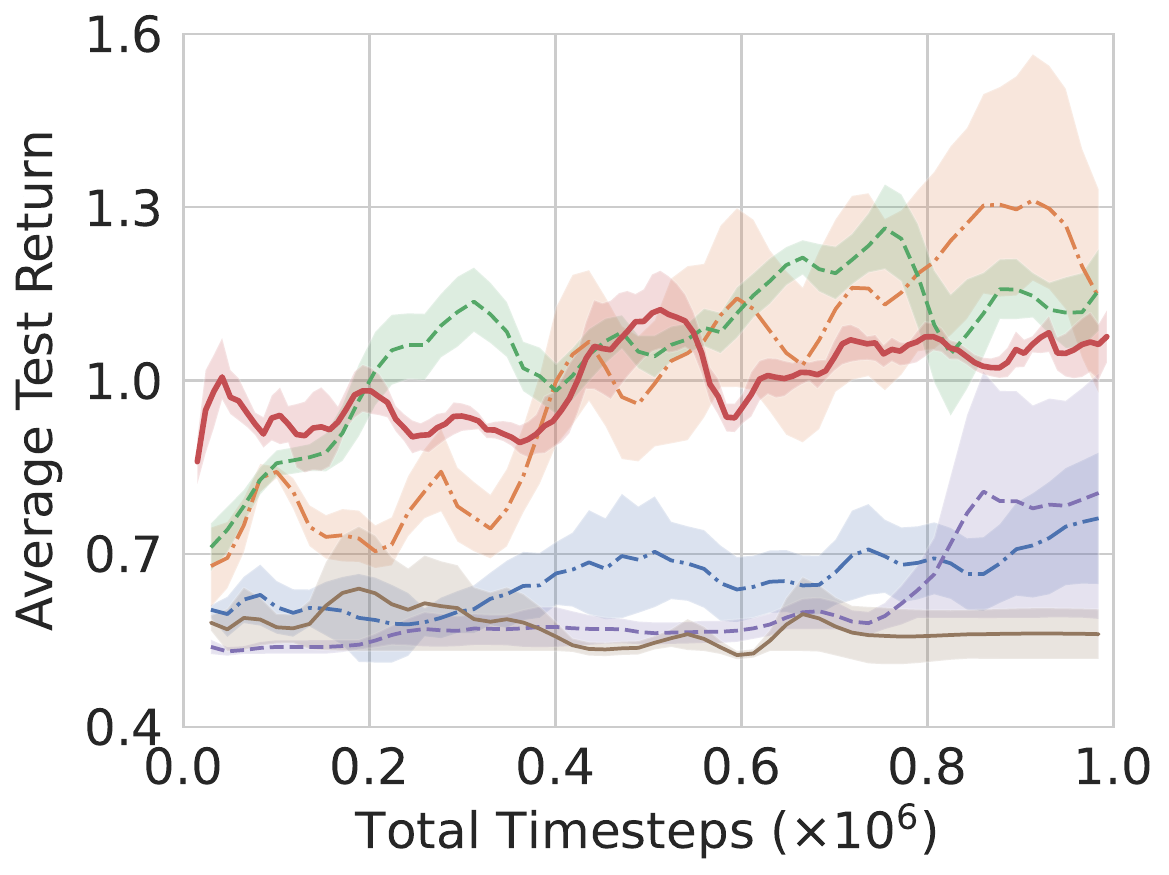}
	}
	\subfloat[][Food-Collection-Sparse]{
	\includegraphics[width=0.33\textwidth]{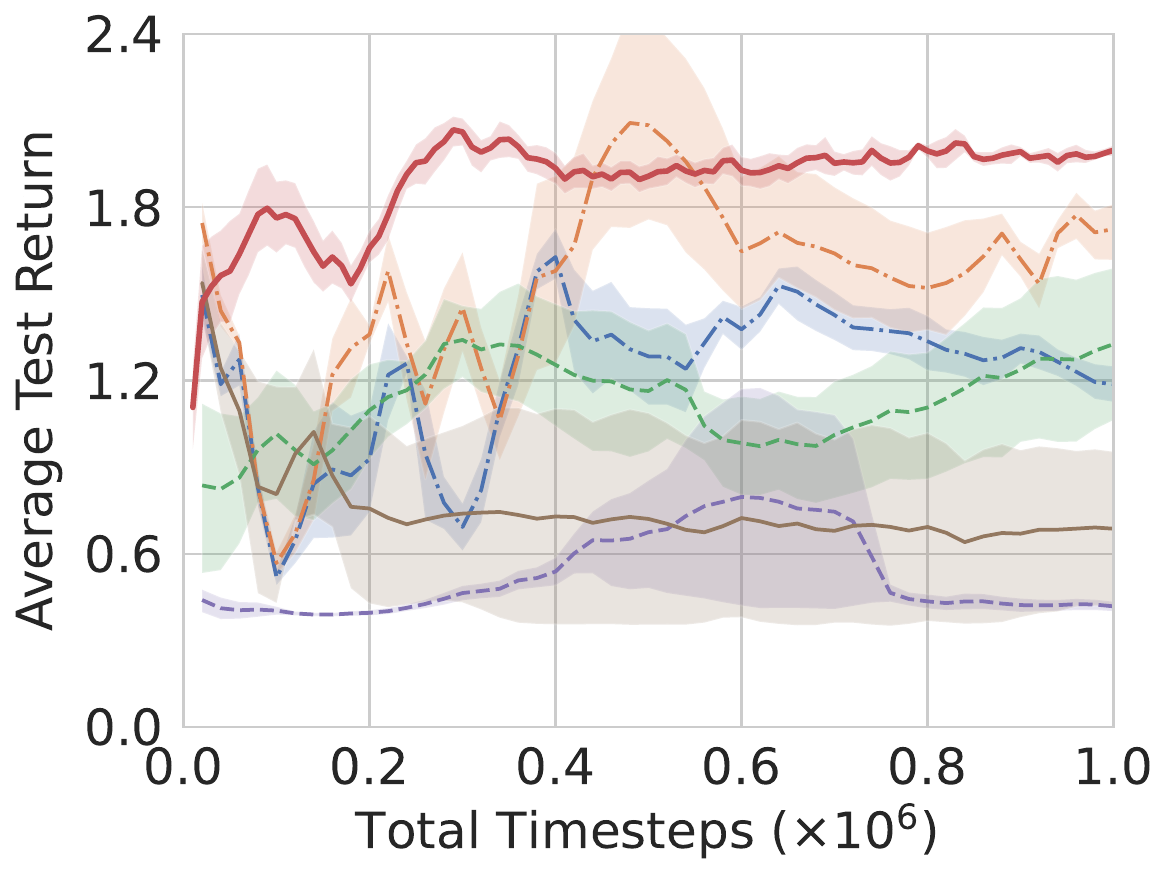}
	}
	\subfloat[][Predator-Prey-Sparse]{
	\includegraphics[width=0.33\textwidth]{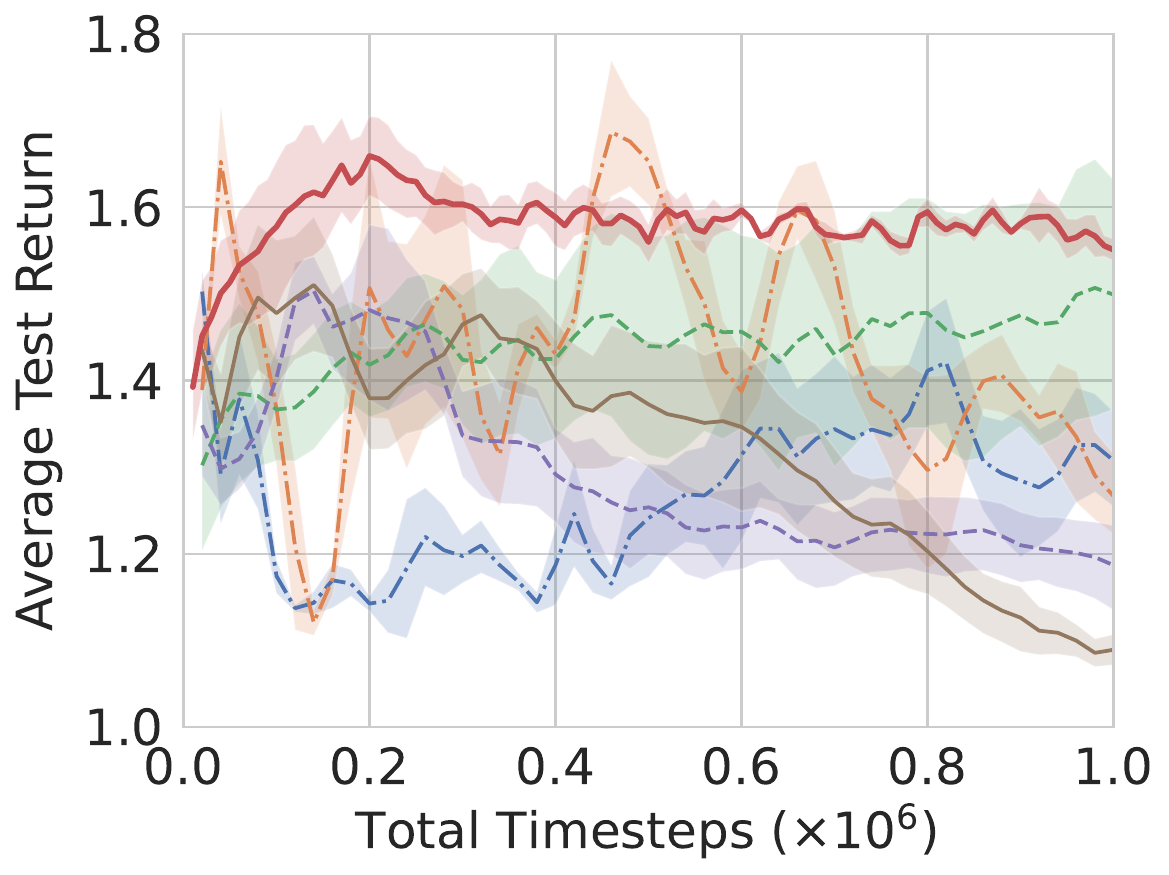}
	}
	
	\subfloat[][Robot-Navigation-Sparse]{
	\includegraphics[width=0.33\textwidth]{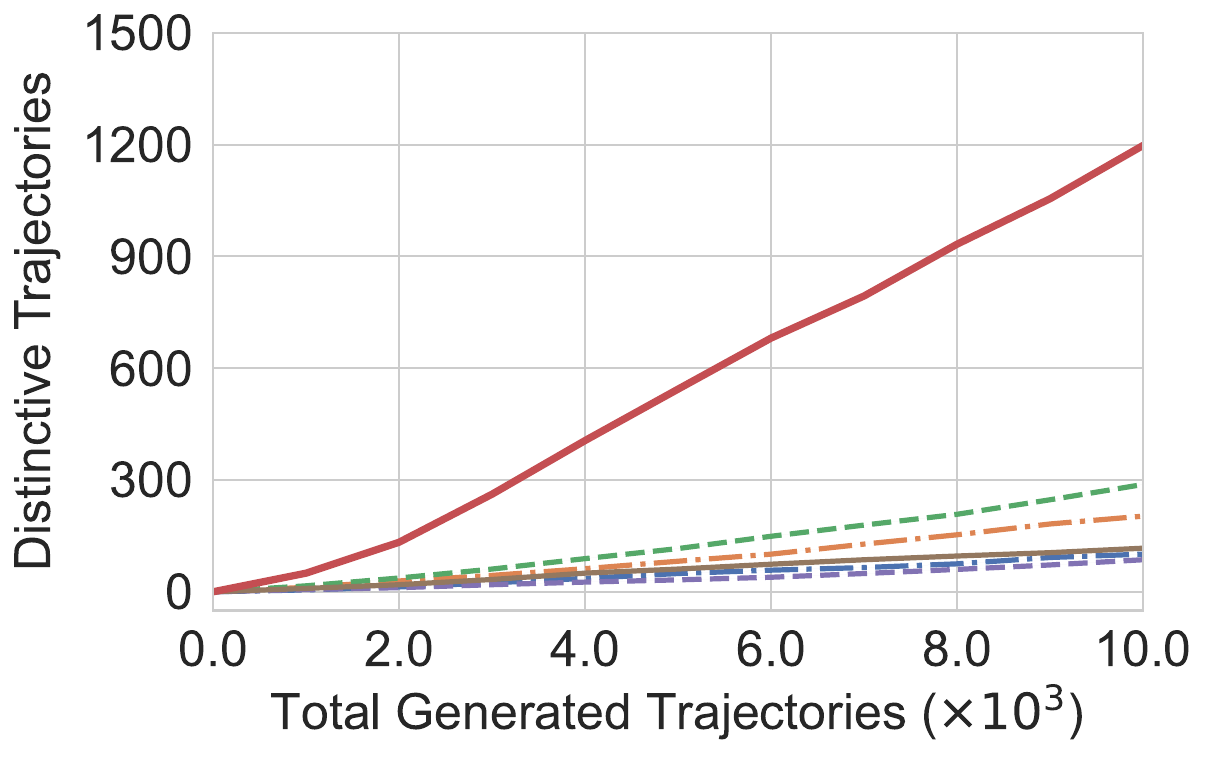}
	}
	\subfloat[][Food-Collection-Sparse]{
	\includegraphics[width=0.33\textwidth]{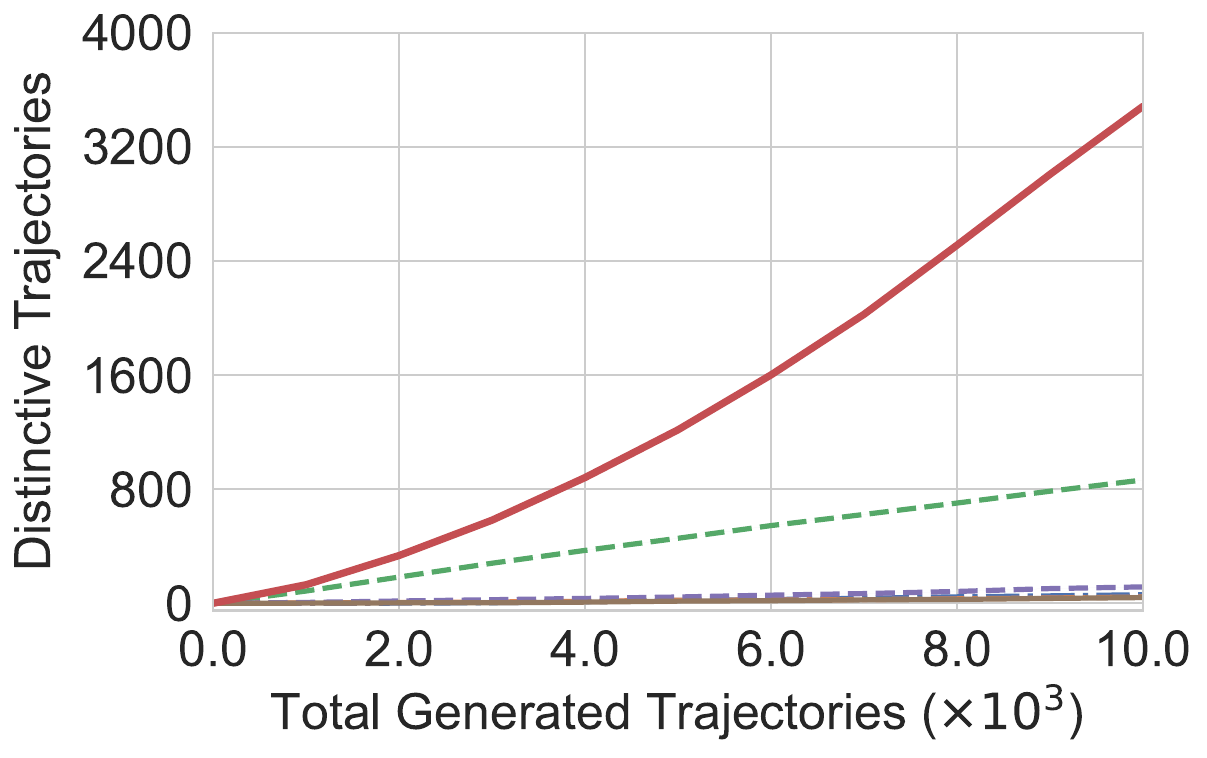}
	}
	\subfloat[][Predator-Prey-Sparse]{
	\includegraphics[width=0.33\textwidth]{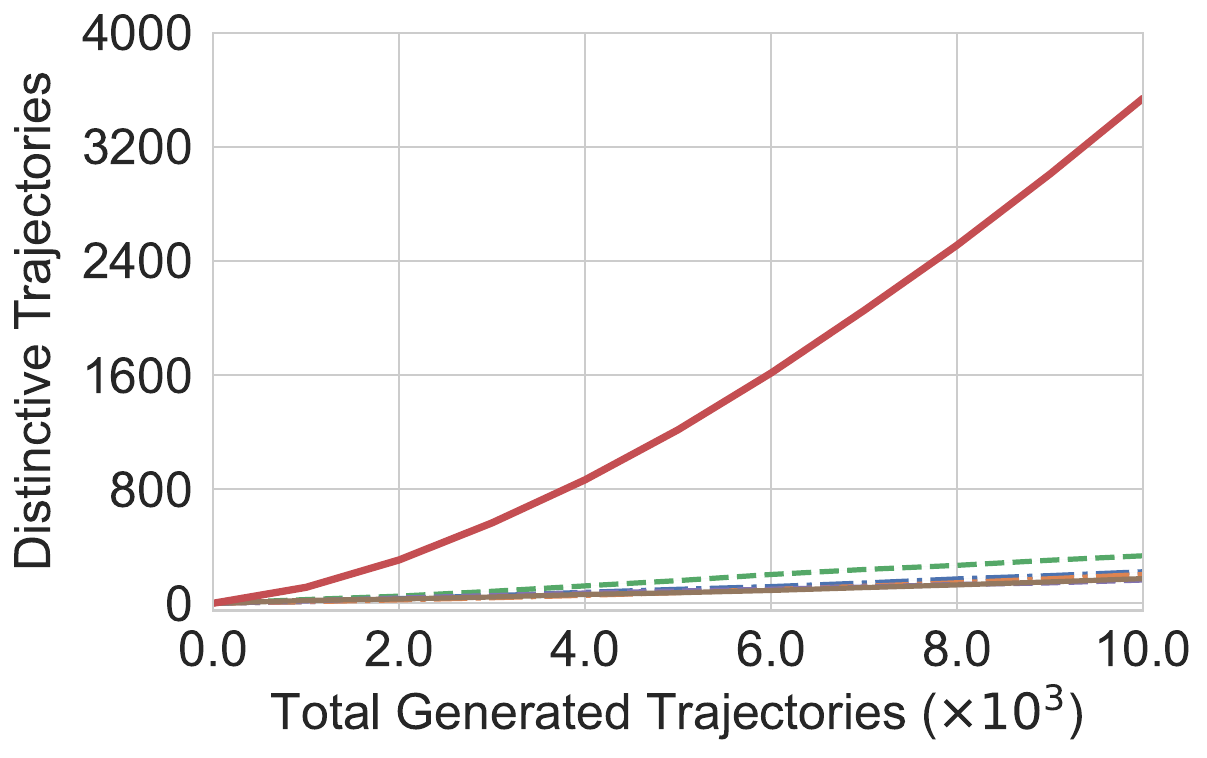}
	}
 	\caption{Comparison results of IDDPG, MADDPG, COVDN, COMIX, FACMAC and MACFN on Robot-Navigation-Sparse~($N=2$), Food-Collection-Sparse~($N=3$), and Predator-Prey-Sparse~($N=3$) scenarios. \textbf{Top:} Average test return of different methods.   \textbf{Bottom:} Number of distinctive trajectories during the training process.} 
	\label{fig_reward}
\end{figure*}

Figure~\ref{fig_reward}(d)-(f) shows the number of distinctive trajectories explored in different MPE scenarios. Compared with the baselines, where the learned policies are more inclined to sample action sequences with the highest rewards, MACFN exceeds the other methods significantly in terms of exploratory capacity. The number of valid trajectories explored by MACFN tends to increase as the training progresses. In contrast, none of the other algorithms could explore diverse trajectories effectively.

\begin{table}[!t]
    \caption{Average test return of our method and baselines on Food-Collection-Sparse environment under different agent number settings.  $\pm$ indicates one standard deviation of the average evaluation over $5$ trials. \textbf{Bold} indicates the best performance~in~each~scenario. }
    \centering
    \label{tab:reuslt1}
    \resizebox{0.66\textwidth}{!}{%
    \begin{tabular}{@{}cccc@{}}
    \toprule
    \multirow{2}{*}{\textbf{Method}} & \multicolumn{3}{c}{\textbf{Food-Collection-Sparse}} \\ \specialrule{0em}{1pt}{1pt} \cmidrule(l){2-4} 
                                     & \textbf{N=3}    & \textbf{N=4}    & \textbf{N=5}    \\ \specialrule{0em}{1pt}{1pt} \midrule
    IDDPG                & 1.19 $\pm$ 0.20  & 1.43 $\pm$ 0.14  & 1.90 $\pm$ 1.04 \\ \specialrule{0em}{1pt}{1pt}
    MADDPG             & 1.72 $\pm$ 0.29  & 2.11 $\pm$ 0.45  & 2.26 $\pm$ 0.81 \\ \specialrule{0em}{1pt}{1pt}
    COVDN                  & 1.33 $\pm$ 0.82  & 1.39 $\pm$ 0.71  & 1.57 $\pm$ 0.77 \\ \specialrule{0em}{1pt}{1pt}
    COMIX                 & 0.42 $\pm$ 0.05  & 0.92 $\pm$ 0.43  & 0.65 $\pm$ 0.17 \\ \specialrule{0em}{1pt}{1pt}
    FACMAC        & 0.69 $\pm$ 0.59  & 0.44 $\pm$ 0.11  & 0.37 $\pm$ 0.03 \\ \specialrule{0em}{1pt}{1pt} \midrule
    \textbf{MACFN (Ours)}           & \textbf{2.00 $\pm$ 0.05}  & \textbf{2.19 $\pm$ 0.12}  & \textbf{2.35 $\pm$ 0.13}    \\ \specialrule{0em}{1pt}{1pt} \bottomrule
    \end{tabular}%
    }
    \end{table}
    
    \begin{table}[!t]
    \caption{Average test return of our method and baselines on Predator-Prey-Sparse environment under different agent number settings.}
    \label{tab:reuslt2}
    \centering
    \resizebox{0.66\textwidth}{!}{%
    \begin{tabular}{@{}cccc@{}}
    \toprule
    \multirow{2}{*}{\textbf{Method}} & \multicolumn{3}{c}{\textbf{Predator-Prey-Sparse}} \\ \specialrule{0em}{1pt}{1pt} \cmidrule(l){2-4} 
                                     & \textbf{N=3}    & \textbf{N=4}    & \textbf{N=5}    \\ \specialrule{0em}{1pt}{1pt} \midrule
    IDDPG               & 1.31 $\pm$ 0.18  & 0.89 $\pm$ 0.05  & 0.73 $\pm$ 0.02           \\ \specialrule{0em}{1pt}{1pt}
    MADDPG            & 1.27 $\pm$ 0.14  & 1.21 $\pm$ 0.25  & 0.98 $\pm$ 0.19          \\ \specialrule{0em}{1pt}{1pt}
    COVDN                & 1.50 $\pm$ 0.33  & 1.12 $\pm$ 0.25  & 0.81 $\pm$ 0.09            \\ \specialrule{0em}{1pt}{1pt}
    COMIX               & 1.19 $\pm$ 0.14  & 0.87 $\pm$ 0.02  & 0.75 $\pm$ 0.05            \\ \specialrule{0em}{1pt}{1pt}
    FACMAC     & 1.09 $\pm$ 0.05  & 0.80 $\pm$ 0.01  & 0.71 $\pm$ 0.03               \\ \specialrule{0em}{1pt}{1pt} \midrule
    \textbf{MACFN (Ours)}             & \textbf{1.55 $\pm$ 0.03}  & \textbf{1.31 $\pm$ 0.08}  & \textbf{1.02 $\pm$ 0.01}      \\ \specialrule{0em}{1pt}{1pt} \bottomrule
    \end{tabular}%
    }
    \end{table}

To further test the generalizability of the proposed MAFCN, we conduct experiments with varying numbers of agents. As the number of agents increases, the interaction between agents and the appropriate flow decomposition often becomes more complicated. 
Despite this challenge, our approach consistently outperforms the baselines with the agent number increases, as shown in Table~\ref{tab:reuslt1} and~\ref{tab:reuslt2}.
Overall, these extensive results clearly demonstrate the advantage of multi-agent flow matching, leading to significantly superior performance compared to the state-of-the-art baselines.

\begin{figure*}[!t]
	\centering
 \subfloat{\quad\;\includegraphics[width=1.0\textwidth]{legend.pdf}}
 \vspace{-0.5em}
 \\    
    \addtocounter{subfigure}{-1}
    \subfloat[][2-Agent-Reacher-Sparse]{
	\includegraphics[width=0.33\textwidth]{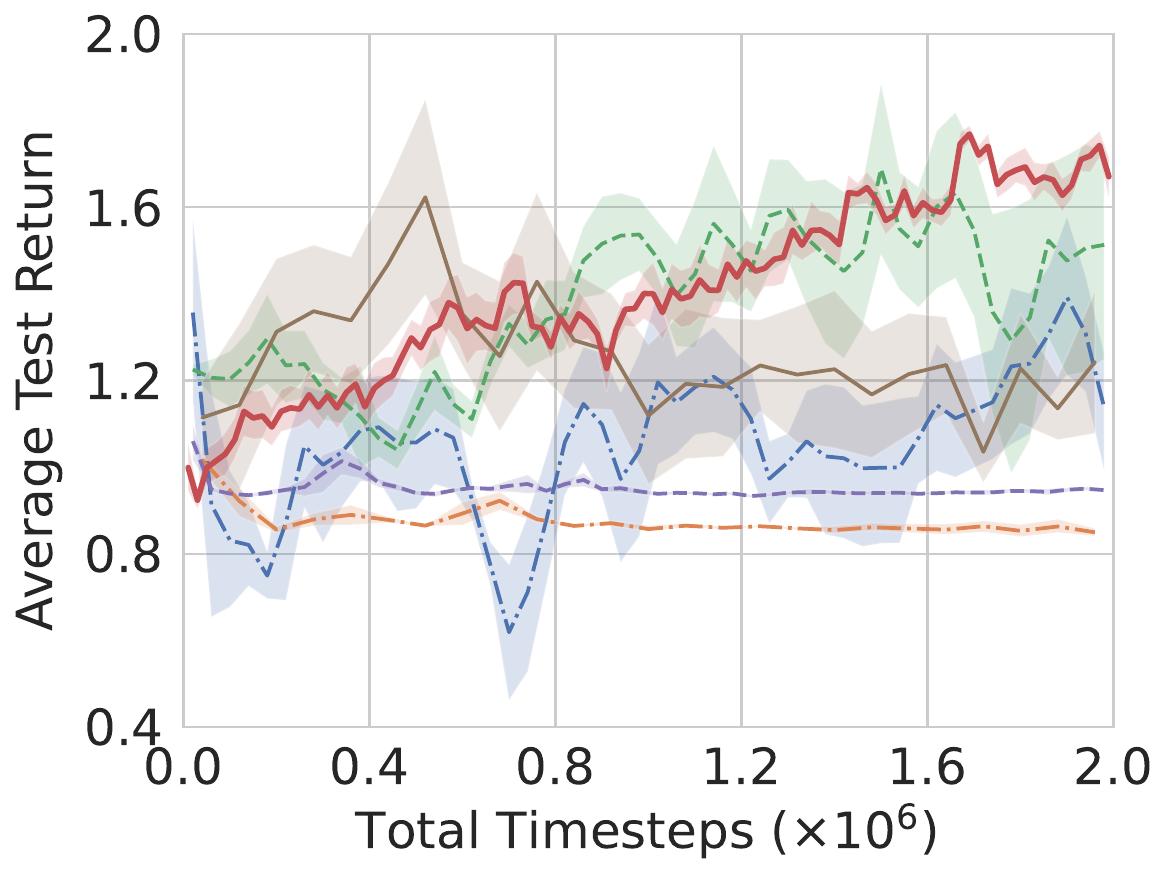}
	}
	\subfloat[][2-Agent-Swimmer-Sparse]{
	\includegraphics[width=0.33\textwidth]{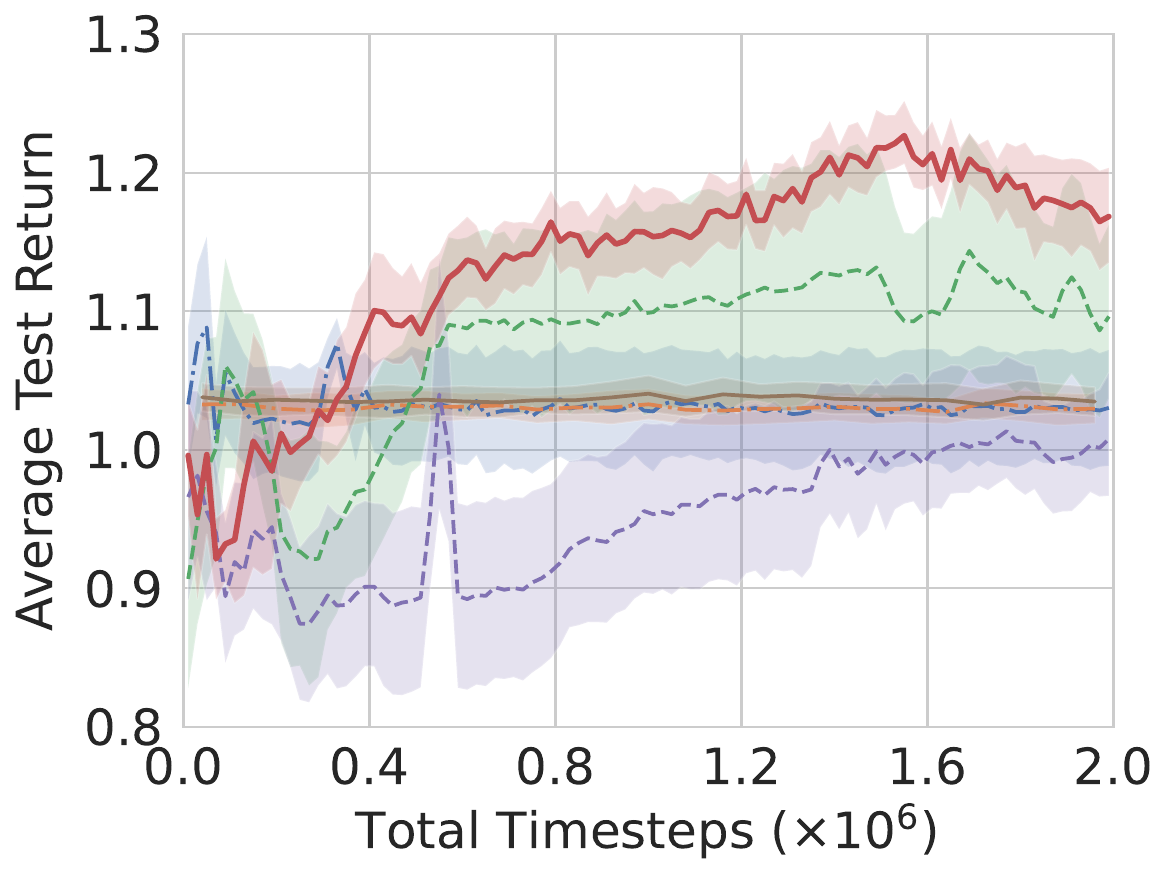}
	}
	\subfloat[][3-Agent-Hopper-Sparse]{
	\includegraphics[width=0.33\textwidth]{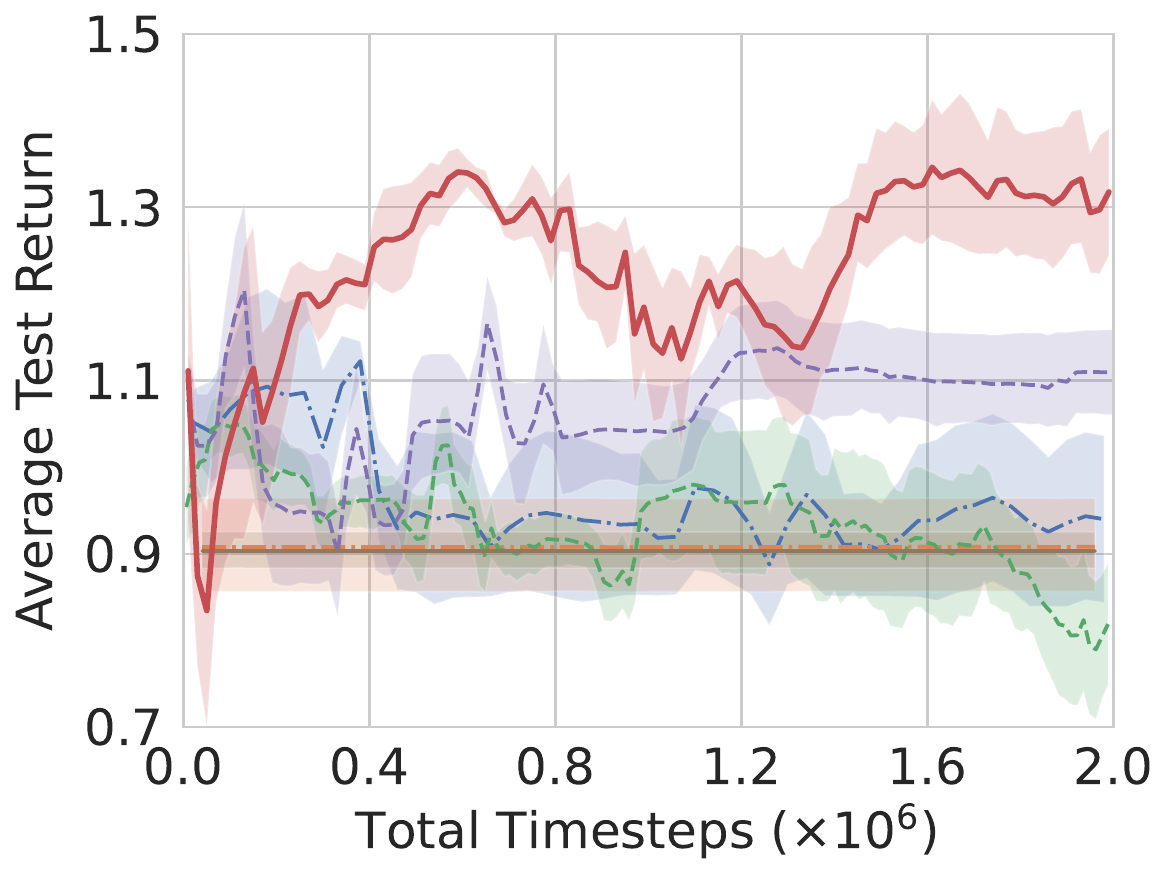}
	}
	
	\subfloat[][2-Agent-Reacher-Sparse]{
	\includegraphics[width=0.33\textwidth]{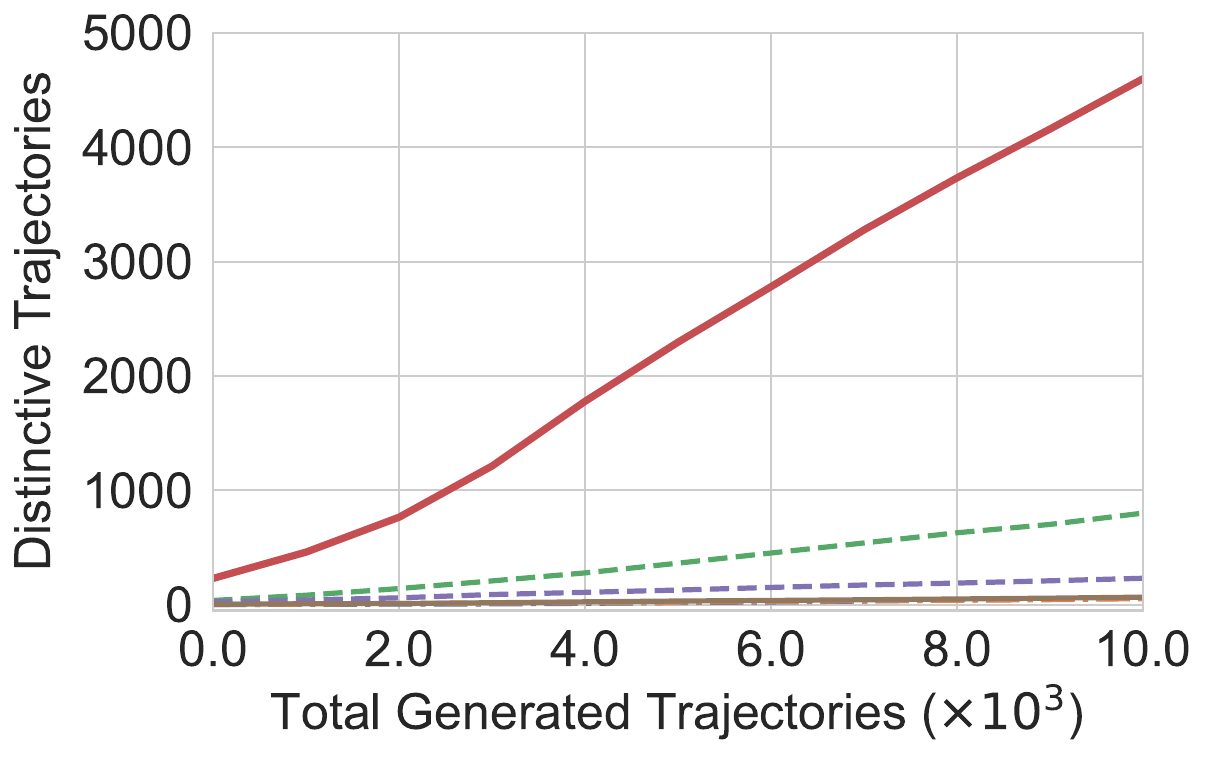}
	}
	\subfloat[][2-Agent-Swimmer-Sparse]{
	\includegraphics[width=0.33\textwidth]{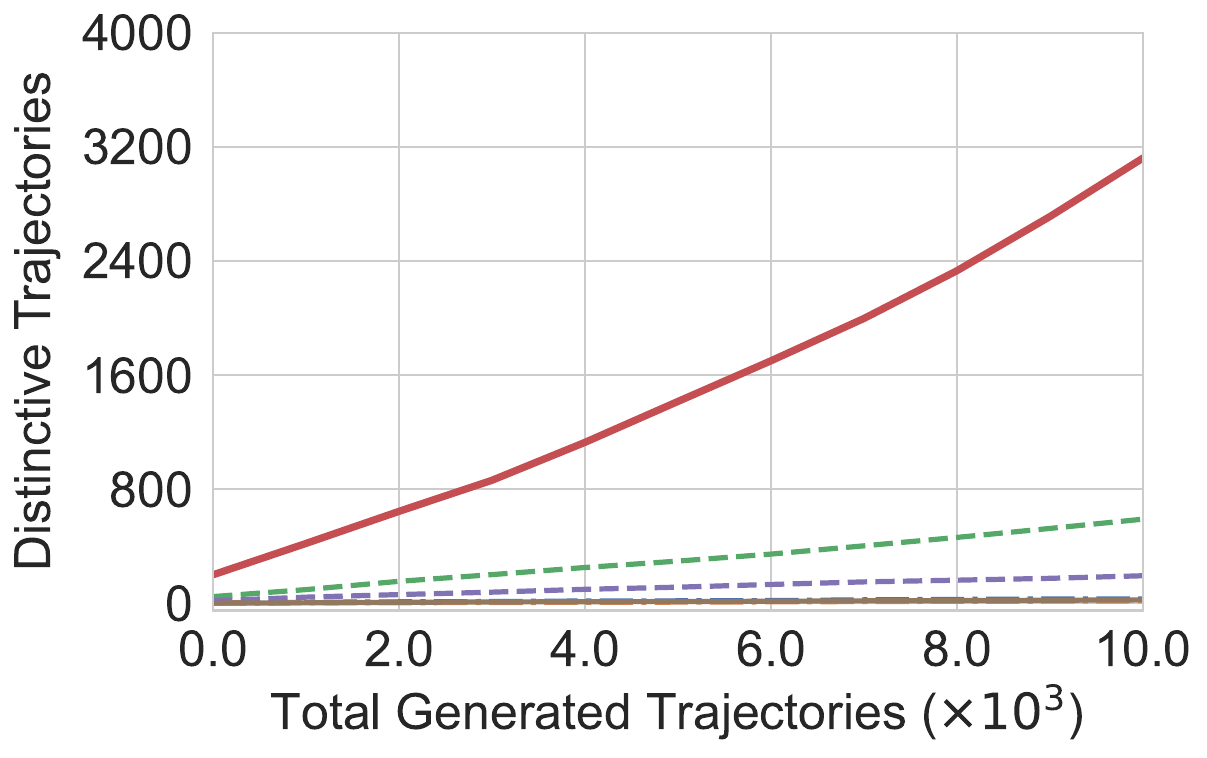}
	}
	\subfloat[][3-Agent-Hopper-Sparse]{
	\includegraphics[width=0.33\textwidth]{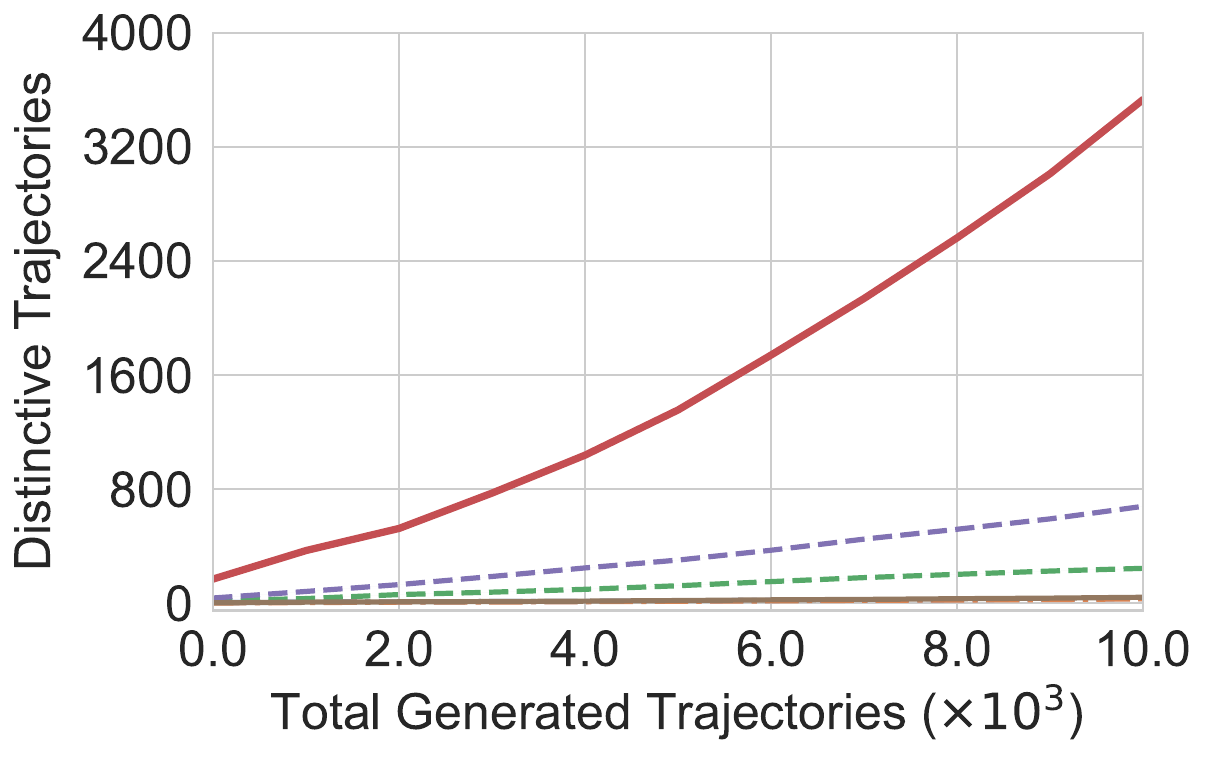}
	}
        \captionsetup{labelfont={color=black}}
 	\caption{Comparison results of IDDPG, MADDPG, COVDN, COMIX, FACMAC and MACFN on 2-Agent-Reacher-Sparse~($N=2$), 2-Agent-Swimmer-Sparse~($N=2$), and 3-Agent-Hopper-Sparse~($N=3$) scenarios. \textbf{Top:} Average test return of different methods.   \textbf{Bottom:} Number of distinctive trajectories during the training process.}
	\label{mujoco_reward}
\end{figure*}

The average test return of different MAMuJoCo scenarios is shown in Figure~\ref{mujoco_reward}(a)-(c).
In all scenarios, our proposed MACFN successfully improves the learning efficiency and the final performance. In the simple scenarios~(\emph{2-Agent-Reacher-Sparse}, \emph{2-Agent-Swimmer-Sparse}), the robot is controlled by two agents. Thus, the baseline, COVDN, can also attain comparable performances. However, in the challenge scenario~(\emph{3-Agent-Hopper-Sparse}), our method consistently exhibits significantly better performance than baselines during training. The results suggest that MACFN can be utilized to explore the diverse reward distributions, which helps the agents construct a more useful policy and achieve non-trivial performance.

Figure~\ref{mujoco_reward}(d)-(f) reports the number of valid-distinctive trajectories explored in different MAMuJoCo scenarios. As observed in the figure, MACFN can offer promising exploration ability, with a significantly higher number of explored trajectories compared to other baselines. While COVDN achieves competitive results of average test return, the trajectories it generates exhibit substantial similarity, indicating limited exploration capacity.
In summary, these findings demonstrate that our proposed MACFN not only exhibits notable performance improvements but also superior exploration ability.

\section{Conclusion}\label{sec:conclusion}
In this work, we develop a theory for generalized multi-agent GFlowNets, involving continuous control. 
Based on the theoretical formulation, we propose a novel training framework, termed MACFN, that enables us to learn decentralized individual-flow-based policies via centralized global-flow-based matching.
The continuous flow decomposition network is at the heart of MACFN, where agents can deduce their flow contributions given only global rewards.
We validate MACFN over several multi-agent control tasks with sparse rewards and showcase that it yields results significantly superior to the state-of-the-art techniques in the continuous setting, especially the exploration capability. To our knowledge, this work is the first attempt to extend GFlowNets to the multi-agent continuous control domain. 

One of the primary limitations of MACFN is its potential inefficiency in situations where only a single optimal solution is sought. Unlike traditional RL agents that are trained to maximize returns, MACFN does not inherently prioritize the most effective single strategy. Instead, it focuses on the broader generation of possible solutions, which may lead to a less direct path toward the optimal solution. MACFN is more suitable for tasks where the goal is to generate a diverse set of high-quality solutions. This can result in increased computational demands and a slower convergence rate when the task does not require diverse strategies.
Thus, MACFN serves as a complementary approach to reinforcement learning, particularly useful in exploratory control tasks.
In our future work, we will improve MACFN under the dynamic multi-agent setting, which is more challenging in adaptive flow matching as the flow number of agents changes over time. Besides, some environments may not satisfy the translation assumption, and it is also an exciting direction to consider more complex transition functions of environments.

\section*{Acknowledge}
This work was supported by the National Key Research and Development Project of China (No.2021ZD0110505), the National Natural Science Foundation of China (No.U19B2042), the Zhejiang Provincial Key Research and Development Project (No.2022C01044), University Synergy Innovation Program of Anhui Province (No.GXXT-2021-004), Academy Of Social Governance Zhejiang University, Fundamental Research Funds for the Central Universities (No.226-2022-00064).

  \bibliographystyle{elsarticle-harv} 
  \bibliography{macfn}

\appendix

\newpage

\renewcommand\thefigure{A\arabic{figure}}
\renewcommand\thetable{A\arabic{table}}
\setcounter{figure}{0}
\setcounter{table}{0}

\section{Abbreviations}

To improve the readability and conciseness, we provide a list of abbreviations throughout the paper in Table~\ref{abbreviations}.

\begin{table}[!h]
\captionsetup{labelfont={color=black}}
\caption{List of abbreviations.}
\label{abbreviations}
\begin{tabular}{@{}lllll@{}}
\toprule
Abbreviation &                                                            &  &  &  \\ \midrule
CTDE          & Centralized Training with Decentralized Execution          &  &  &  \\
DAG           & Directed Acyclic Graph                                     &  &  &  \\
Dec-POMDP     & Decentralized Partially Observable Markov Decision Process &  &  &  \\
GFlowNets     & Generative Flow Networks                                   &  &  &  \\
IDDPG         & Independent DDPG                                           &  &  &  \\
MACFN         & Multi-Agent generative Continuous Flow Networks            &  &  &  \\
MARL            & Multi-Agent Reinforcement Learning                       &  &  &  \\ 
MAS           & Multi-Agent Systems                                        &  &  &  \\
MPE           & Multi-Agent Particle Environment                           &  &  &  \\
PyMARL        & Python MARL framework                                      &  &  &  \\ 
RL            & Reinforcement Learning                                     &  &  &  \\  \bottomrule
\end{tabular}
\end{table}

\section{Proofs}
\subsection{Proof of Lemma~1}\label{PL1}
\noindent\textbf{Lemma~1.} {\em
    Let $\pi(\boldsymbol{a_t} \mid s_t)=\frac{F\left(s_t, \boldsymbol{a}_t\right)}{F\left(s_t\right)}$ denotes the joint policy, and $\pi_i\left(a_i \mid o_i\right)$ denotes the individual policy of agent $i$. 
    Under Definition~\ref{def-decomposition}, we have
    \begin{align}
    \pi(\boldsymbol{a} \mid s) =\prod_{i=1}^N \pi_i\left(a_i \mid o_i\right).
    \end{align}
}

\begin{proof} 

    First, for any $s \neq s_f$ and $s \neq s_0$, the set of complete trajectories passing $s$ is the disjoint union of the sets of trajectories passing $s \rightarrow s'$. 
    According to the definition, the state flow
    \begin{equation}
    F(s_t)= \int_{\boldsymbol{a}_t \in \mathcal{A}} F(s_t,\boldsymbol{a}_t) \mathrm{d} \boldsymbol{a}_t.    
    \end{equation}
    Similarly, the observation flow of agent $i$ can defined as 
    \begin{equation}
    F(o_t^i)= \int_{a_t^i \in \mathcal{A}} F(o_t^i,a_t^i) \mathrm{d} a_t^i.   
    \end{equation}
    Since the flow function $F(o_t^i)$ of each agent are independent, we have
    \begin{align}
        \prod_{i=1}^N F_i(o_t^i) &=\prod_{i=1}^N \int_{a_t^i \in \mathcal{A}} F(o_t^i,a_t^i) \mathrm{d} a_t^i 
        \\&= \int_{a_t^1 \in \mathcal{A}} F(o_t^1,a_t^1) \mathrm{d} a_t^1 \cdots \int_{a_t^N \in \mathcal{A}} F(o_t^N,a_t^N) \mathrm{d} a_t^N  
        \\&= \int_{(a_t^1,\cdots,a_t^N) \in \mathcal{A}^1 \times \cdots \times \mathcal{A}^N} F(o_t^1,a_t^1) \cdots F(o_t^N,a_t^N)
        \mathrm{d} a_t^1\cdots a_t^N 
        \\&= \int_{\boldsymbol{a}_t \in \mathcal{A}} F(s_t,\boldsymbol{a}_t) \mathrm{d} \boldsymbol{a}_t = F(s_t). 
    \end{align}
    Therefore, the joint policy 
    \begin{align}
    \pi(\boldsymbol{a} \mid s)=\frac{F\left(s_t, \boldsymbol{a}_t\right)}{F\left(s_t\right)}=\frac{\prod _{i=1}^N F_i\left(o_t^i, a_t^i\right)}{\prod_{i=1}^N F_i(o_t^i)} =\prod_{i=1}^N \pi_i\left(a_i \mid o_i\right).
    \end{align}
    Then we complete the proof.
\end{proof}

\subsection{Proof of Lemma~2}\label{PL2}
\noindent\textbf{Lemma~2.} {\em
For any state $s_{t}$, its outflows and inflows are calculated as follows: 
\begin{align}
    &\int_{s\in \mathcal{C} (s_t)}  F(s_t \rightarrow s) \mathrm{d} s = \int_{\boldsymbol{a} \in \mathcal{A}} F(s_t,\boldsymbol{a}_t) \mathrm{d} \boldsymbol{a}_t 
    \\&= \int_{\boldsymbol{a} \in \mathcal{A}} \prod_{i=1}^N F_i(o_t^i, a_t^i) \mathrm{d} \boldsymbol{a}_t
    = \prod_{i=1}^N \int_{a_t^i \in \mathcal{A}^i} F_i(o_t^i, a_t^i) \mathrm{d} a_t^i,
\end{align}
and
\begin{align}
    &\int_{s\in \mathcal{P} (s_t)}  F(s\rightarrow s_t) \mathrm{d} s  = \int_{\boldsymbol{a}:T(s,\boldsymbol{a})= s_{t}} F(s,\boldsymbol{a}) \mathrm{d} \boldsymbol{a} \\
    &=\int_{\boldsymbol{a}:T(s,\boldsymbol{a})= s_{t}} \prod_{i=1}^N F_i(o^i, a^i) \mathrm{d} \boldsymbol{a} 
    =\prod_{i=1}^N\int_{a^i:T(o^i,a^i)= o_t^i} F_i(G_\phi(o_t^i,a^i), a^i) \mathrm{d} a^i,
\end{align}
where $\boldsymbol{a}$ is the unique action that transition to $s_t$ from $s$ and $o^i=G_\phi(o_t^i,a^i)$ with $o_t^i=T(o^i,a^i)$.
}

\begin{proof}
According to Assumption~\ref{assumption0}, for any state pair $(s_t,s_{t+1})$, there is a unique joint action $\boldsymbol{a_t}$ such that $T(s_t,\boldsymbol{a_t}) = s_{t+1}$ and the action is the translation action, so we can represent outflows in Definition~\ref{def-outflows} and inflows in Definition~\ref{def-inflows} as an integral over actions as
\begin{align}
    \int_{s\in \mathcal{C} (s_t)}  F(s_t \rightarrow s) \mathrm{d} s = \int_{\boldsymbol{a} \in \mathcal{A}} F(s_t,\boldsymbol{a}_t) \mathrm{d} \boldsymbol{a}_t,
\end{align}
and
\begin{equation}
\int_{s\in \mathcal{P} (s_t)}  F(s\rightarrow s_t) \mathrm{d} s = \int_{\boldsymbol{a}:T(s,\boldsymbol{a})= s_{t}} F(s,\boldsymbol{a}) \mathrm{d} \boldsymbol{a}.
\end{equation}
According to Definition~\ref{def-decomposition}, we can decompose the above integral into the integrals over the actions of individual agents as
\begin{align}
    \int_{\boldsymbol{a} \in \mathcal{A}} F(s_t,\boldsymbol{a}_t) \mathrm{d} \boldsymbol{a}_t 
    = \int_{\boldsymbol{a} \in \mathcal{A}} \prod_{i=1}^N F_i(o_t^i, a_t^i) \mathrm{d} \boldsymbol{a}_t,
\end{align}
and
\begin{align}
    \int_{\boldsymbol{a}:T(s,\boldsymbol{a})= s_{t}} F(s,\boldsymbol{a}) \mathrm{d} \boldsymbol{a} 
    &=\int_{\boldsymbol{a}:T(s,\boldsymbol{a})= s_{t}} \prod_{i=1}^N F_i(o^i, a^i) \mathrm{d} \boldsymbol{a}.
\end{align}
Note that $\mathcal{A} = \mathcal{A}^1 \times \cdots \times \mathcal{A}^N$ and $o^i=G_\phi(o_t^i,a^i)$. So we have
\begin{align}
    \int_{\boldsymbol{a} \in \mathcal{A}} \prod_{i=1}^N F_i(o_t^i, a_t^i) \mathrm{d} \boldsymbol{a}_t
    = \prod_{i=1}^N \int_{a_t^i \in \mathcal{A}^i} F_i(o_t^i, a_t^i) \mathrm{d} a_t^i,
\end{align}
and
\begin{align}
    \int_{\boldsymbol{a}:T(s,\boldsymbol{a})= s_{t}} \prod_{i=1}^N F_i(o^i, a^i) \mathrm{d} \boldsymbol{a}
    =\prod_{i=1}^N\int_{a^i:T(o^i,a^i)= o_t^i} F_i(G_\phi(o_t^i,a^i), a^i) \mathrm{d} a^i.
\end{align}
Then we complete the proof.
\end{proof}

\subsection{Proof of Lemma~4}\label{PL5}
\noindent\textbf{Lemma~4.} {\em
Let $\{a^{1,k},...,a^{N,k}\}_{k=1}^{K}$ be sampled independently and uniformly from the continuous action space $\mathcal{A}^1 \times \cdots \times \mathcal{A}^N$.
Assume $G_{\phi^\star}$ can optimally output the actual state $o_t^i$ with $(o_{t+1}^i,a_t^i)$. Then for any state $s_t \in \mathcal{S}$, we have
\begin{equation}
    \mathbb{E}\left[\frac{\mu(\mathcal{A})}{K} \sum_{k=1}^K \prod_{i=1}^N  F(o_t^{i},a_t^{i,k}) \right]
    = \int_{\boldsymbol{a} \in \mathcal{A}}F(s_t,\boldsymbol{a})\mathrm{d} \boldsymbol{a}
\end{equation}
and
\begin{equation}
    \mathbb{E}\bigg[\frac{\mu(\mathcal{A})}{K} \sum_{k=1}^K \prod_{i=1}^N F(G_{\phi^\star} (o_t^{i}, a_t^{i,k}),a_t^{i,k})\bigg]
    =  \int_{\boldsymbol{a}:T(s,\boldsymbol{a})=s_t}F(s,\boldsymbol{a}) \mathrm{d} \boldsymbol{a}.
\end{equation}
}

\begin{proof}
    
    Since $\{a^{1,k},...,a^{N,k}\}_{k=1}^{K}$ is sampled independently and uniformly from the continuous action space $\mathcal{A}^1 \times \cdots \times \mathcal{A}^N$, we have
    \begin{equation}
        \mathbb{E}\left[ F(s_t,\boldsymbol{a}_t^k)\right]
        = \frac{1}{\mu(\mathcal{A})}\int_{\boldsymbol{a} \in \mathcal{A}}F(s_t,\boldsymbol{a})\mathrm{d} \boldsymbol{a}.
    \end{equation}
    Hence
    \begin{align}
        \mathbb{E}\left[\frac{\mu(\mathcal{A})}{K} \sum_{k=1}^K \prod_{i=1}^N  F(o_t^{i},a_t^{i,k}) \right] 
        = \frac{\mu(\mathcal{A})}{K}  \sum_{k=1}^K  \mathbb{E}\left[ F(s_t,\boldsymbol{a}_t^k)\right]
        = \int_{\boldsymbol{a} \in \mathcal{A}}F(s_t,\boldsymbol{a})\mathrm{d} \boldsymbol{a}.
    \end{align}
    Note that for any pair of $(s,\boldsymbol{a})$ satisfying $T(s,\boldsymbol{a})=s_t$, $s$ is unique if we fix $\boldsymbol{a}$. So we have
    \begin{align}
        \mathbb{E}\left[F(G_{\phi^\star} (s_t, \boldsymbol{a}_k), \boldsymbol{a}_k)\right] =\mathbb{E}\left[F(s_{t-1}, \boldsymbol{a}_k)\right] 
        = \mathbb{E}\left[\prod_{i=1}^N F(o_{t-1}^i, a_t^{i,k})\right]
    \end{align}
    and
    \begin{equation}
        \mathbb{E}\left[F(G_{\phi^\star} (s_t, \boldsymbol{a}_k), \boldsymbol{a}_k)\right]
        =  \frac{1}{\mu(\mathcal{A})} \int_{\boldsymbol{a}:T(s,\boldsymbol{a})=s_t}F(s,\boldsymbol{a}) \mathrm{d} \boldsymbol{a}.
    \end{equation}
    Since $o_t^i = G_{\phi^\star}(s_t,i)$, we have
    
    \begin{equation}
        \mathbb{E}\left[\frac{\mu(\mathcal{A})}{K} \sum_{k=1}^K \prod_{i=1}^N F(G_{\phi^\star}(o_t^{i}, a_t^{i,k}),a_t^{i,k})\right]
        =  \int_{\boldsymbol{a}:T(s,\boldsymbol{a})=s_t}F(s,\boldsymbol{a}) \mathrm{d} \boldsymbol{a}.
    \end{equation}
    Then we complete the proof.
\end{proof}

\subsection{Proof of Theorem~1}\label{PT2}
\noindent\textbf{Theorem~1.} {\em
Let $\{a^{1,k},...,a^{N,k}\}_{k=1}^{K}$ be sampled independently and uniformly from the continuous action space $\mathcal{A}^1 \times \cdots \times \mathcal{A}^N$. Assume $G_{\phi^\star}$ can optimally output the actual state $o_t^i$ with $(o_{t+1}^i,a_t^i)$. For any bounded continuous action $\boldsymbol{a} \in \mathcal{A}$, any state $s_t \in \mathcal{S}$ and any $\delta>0$, we have
    \begin{multline} 
        \mathbb{P}\left(\Big{|}\frac{\mu(\mathcal{A})}{K} \sum_{k=1}^K \prod_{i=1}^N  F(o_t^{i},a_t^{i,k}) -\int_{\boldsymbol{a} \in \mathcal{A}}F(s_t,\boldsymbol{a})\mathrm{d} \boldsymbol{a} \Big{|}\right.  \\  \left. \ge \frac{\delta L\mu(\mathcal{A}) \rm{diam}(\mathcal{A}) }{\sqrt{K}} \right)
        \le 2\exp\left(-\frac{\delta^2}{2}\right)
    \end{multline}
    and 
    \begin{multline}
       \mathbb{P}\left( \Big{|}\frac{\mu(\mathcal{A})}{K} \sum_{k=1}^K \prod_{i=1}^N F(G_{\phi} (o_t^{i}, a_t^{i,k}),a_t^{i,k})    - \int_{\boldsymbol{a}:T(s,\boldsymbol{a})=s_t}F(s,a) \mathrm{d} \boldsymbol{a} \Big{|} \right.  \\  \left.
         \ge \frac{\delta L\mu(\mathcal{A}) [\rm{diam}(\mathcal{A}) + \rm{diam}(\mathcal{S})]}{\sqrt{K}} + \frac{\mu(\mathcal{A}) N^\alpha}{K^\beta} \right)\\ 
       \le 2\exp\left(-\frac{\delta^2}{2}\right),
    \end{multline}
    where $L$ is the Lipschitz constant of the function $F(s_t,a)$, $\rm{diam}(\mathcal{A})$ denotes the diameter of the action space $\mathcal{A}$ and $\mu(\mathcal{A})$ denotes the measure of the action space $\mathcal{A}$.
}

\begin{proof} 
      
      Based on Lemma~\ref{lm:expectation}, we define
      \begin{align}
          \Gamma_k &= \frac{\mu(\mathcal{A})}{K} F(s_t,{\boldsymbol{a}_t^k})-\frac{1}{K} \int_{\boldsymbol{a} \in \mathcal{A}}F(s_t,\boldsymbol{a})\mathrm{d} \boldsymbol{a} \\&= \frac{\mu(\mathcal{A}^1)\cdots\mu(\mathcal{A}^N)}{K}  F(s_t,{\boldsymbol{a}_t^k})-\frac{1}{K} \int_{\boldsymbol{a} \in \mathcal{A}} F(s_t,\boldsymbol{a})\mathrm{d} \boldsymbol{a} \\&=
          \frac{\prod_{i=1}^N\mu(\mathcal{A}^i) \cdot F(o_t^{i},a_t^{i,k})}{K}-\frac{1}{K} \int_{\boldsymbol{a} \in \mathcal{A}}F(s_t,\boldsymbol{a})\mathrm{d} \boldsymbol{a}
           \\&= \frac{1}{K} \int_{\boldsymbol{a} \in \mathcal{A}}\left[ \prod_{i=1}^N F(o_t^{i},a_t^{i,k}) - F(s_t,\boldsymbol{a})\right]\mathrm{d} \boldsymbol{a} 
      \end{align}
      and
      \begin{align}
          \Lambda_k&= \frac{\mu(\mathcal{A})}{K}  F(G_{\phi^\star} (s_t, \boldsymbol{a}_k), \boldsymbol{a}_k)-\frac{1}{K} \int_{\boldsymbol{a}:T(s,\boldsymbol{a})=s_t}F(s,\boldsymbol{a}) \mathrm{d} \boldsymbol{a}\\
          &= \frac{\mu(\mathcal{A}_1)\cdots\mu(\mathcal{A}_N)}{K}  F(G_{\phi^\star} (s_t, \boldsymbol{a}_k), \boldsymbol{a}_k)-\frac{1}{K} \int_{\boldsymbol{a}:T(s,\boldsymbol{a})=s_t}F(s,\boldsymbol{a}) \mathrm{d} \boldsymbol{a} \\&= \frac{\prod_{i=1}^N\mu(\mathcal{A}_i) \cdot F(G_{\phi^\star} (o_t^i, a_t^{i,k}),a_t^{i,k})}{K}-\frac{1}{K}\int_{\boldsymbol{a}:T(s,\boldsymbol{a})=s_t}F(s,\boldsymbol{a}) \mathrm{d} \boldsymbol{a}\\&=
           \frac{1}{K} \int_{\boldsymbol{a}:T(s,\boldsymbol{a})=s_t} \left[ \prod_{i=1}^N F(G_{\phi^\star} (o_t^i, a_t^{i,k}),a_t^{i,k}) - F(s,\boldsymbol{a})\right] \mathrm{d} \boldsymbol{a},
      \end{align}
      which yields
      \begin{multline}
          \mathbb{P}\left(\Big{|}\frac{\mu(\mathcal{A})}{K} \sum_{k=1}^K \prod_{i=1}^N  F(o_t^{i},a_t^{i,k})-\int_{\boldsymbol{a} \in \mathcal{A}}F(s_t,\boldsymbol{a})\mathrm{d} \boldsymbol{a} \Big{|} \ge t \right) \\= \mathbb{P}\left( \Big{|} \sum_{k=1}^{K} \Gamma_k \Big{|} \ge t\right)
      \end{multline}
      and 
      \begin{multline}
          \mathbb{P}\left( \Big{|}\frac{\mu(\mathcal{A})}{K} \sum_{k=1}^K \prod_{i=1}^N F(G_{\phi^\star} (o_t^{i}, a_t^{i,k}),a_t^{i,k}) -\int_{\boldsymbol{a}:T(s,\boldsymbol{a})=s_t}F(s,a) \mathrm{d} \boldsymbol{a} \Big{|} \ge t \right) \\= \mathbb{P}\left( \Big{|} \sum_{k=1}^{K} \Lambda_k \Big{|} \ge t\right)
      \end{multline}
      for any $t > 0$.
      Notice that the variables $\{\Gamma_k\}_{k=1}^{K}$ are independent and $\mathbb{E}[\Gamma_k]=0, k=1,\ldots,K$. Using the fact that $F(s,a)$ is a Lipschitz function, we have
      \begin{align}
          |\Gamma_k|&\le 
          \frac{1}{K} \int_{\boldsymbol{a} \in \mathcal{A}}\left[ \prod_{i=1}^N F(o_t^{i},a_t^{i,k}) - F(s_t,\boldsymbol{a})\right]\mathrm{d} \boldsymbol{a} 
          \\
          & =
          \frac{1}{K} \int_{\boldsymbol{a} \in \mathcal{A}}\left[ F(s_t,\boldsymbol{a}^k_t) - F(s_t,\boldsymbol{a})\right]\mathrm{d} \boldsymbol{a}  \\
          & \le \frac{L}{K}  \int_{a \in \mathcal{A}}|| \boldsymbol{a}^k_t-\boldsymbol{a}||\mathrm{d} \boldsymbol{a}\\
          & \le  \frac{L \mu(\mathcal{A})\rm{diam}(\mathcal{A})}{K}.
      \end{align}
      Since for any pair of $(s, \boldsymbol{a})$ satisfying $T(s,\boldsymbol{a})=s_t$, $s$ is unique if we fix $\boldsymbol{a}$,~we~have
      \begin{align}
          |\Lambda_k| & \le 
          \frac{1}{K} \int_{\boldsymbol{a}:T(s,\boldsymbol{a})=s_t} \left[ \prod_{i=1}^N F(G_{\phi^\star} (o_t^i, a_t^{i,k}),a_t^{i,k}) - F(s,\boldsymbol{a})\right] \mathrm{d} \boldsymbol{a}  \\
          & = \frac{1}{K} \int_{\boldsymbol{a}:T(s,\boldsymbol{a})=s_t} \big{|}F(G_{\phi^\star} (s_t, \boldsymbol{a}^k_t), \boldsymbol{a}^k_t)-F(s,\boldsymbol{a})\big{|} \mathrm{d} \boldsymbol{a}  \\
          & \le \frac{1}{K} \int_{\boldsymbol{a}:T(s,\boldsymbol{a})=s_t} \big{|}F(G_{\phi^\star} (s_t, \boldsymbol{a}^k_t), \boldsymbol{a}^k_t)
          -F(s,\boldsymbol{a}^k_t)
          +F(s,\boldsymbol{a}^k_t)
          -F(s,\boldsymbol{a})
          \big{|} \mathrm{d} \boldsymbol{a}  \\
          & \le \frac{L}{K} \int_{\boldsymbol{a}:T(s,a)=s_t} 
          || G_{\phi^\star} (s_t, \boldsymbol{a}^k_t)-s ||+
          ||\boldsymbol{a}_k-\boldsymbol{a} ||\mathrm{d} \boldsymbol{a} \\
          & \le \frac{L \mu(\mathcal{A})[\rm{diam}(\mathcal{A})+\rm{diam}(\mathcal{S})]}{K}.
      \end{align}
      
      \begin{lemma}[Hoeffding's inequality,~\citet{vershynin2018high}] 
        \label{lm:hoeffding}
        Let $x_1,\ldots,x_K$ be independent random variables. Assume the variables $\{x_k\}_{k=1}^{K}$ are bounded in the interval $[T_l,T_r]$. Then for any $t>0$, we have
      \begin{equation}
          \mathbb{P}\left( \Big{|}\sum_{k=1}^{K}(x_k-\mathbb{E}x_k)\Big{|} \ge t \right) \le 2\exp\left(-\frac{2t^2}{K(T_r-T_l)^2}\right).
      \end{equation}
        \end{lemma}
        
      Together with Lemma~\ref{lm:hoeffding} by respectively setting $T_r={L \mu(\mathcal{A})\rm{diam}(\mathcal{A})}/{K}$, $T_l=-{L \mu(\mathcal{A})\rm{diam}(\mathcal{A})}/{K}$ and $T_r={L \mu(\mathcal{A})[\rm{diam}(\mathcal{A})+\rm{diam}(\mathcal{S})]}/{K}$, $T_l=-{L \mu(\mathcal{A})[\rm{diam}(\mathcal{A})+\rm{diam}(\mathcal{S})]}/{K}$, we obtain
      \begin{multline} \label{eq: nophi}
          \mathbb{P}\left(\Big{|}\frac{\mu(\mathcal{A})}{K} \sum_{k=1}^K \prod_{i=1}^N  F(o_t^{i},a_t^{i,k})-\int_{\boldsymbol{a} \in \mathcal{A}}F(s_t,\boldsymbol{a})\mathrm{d} \boldsymbol{a} \Big{|} \ge t \right) \\
          \le 2\exp\left(-\frac{Kt^2}{2(L \mu(\mathcal{A})\rm{diam}(\mathcal{A}))^2}\right)
      \end{multline}
      and 
      \begin{multline} \label{eq: phi_star}
          \mathbb{P}\left( \Big{|}\frac{\mu(\mathcal{A})}{K} \sum_{k=1}^K \prod_{i=1}^N F(G_{\phi^\star} (o_t^{i}, a_t^{i,k}),a_t^{i,k}) -\int_{\boldsymbol{a}:T(s,\boldsymbol{a})=s_t}F(s,a) \mathrm{d} \boldsymbol{a} \Big{|} 
           \ge t \right)  \\
           \le 2\exp\left(-\frac{Kt^2}{2(L \mu(\mathcal{A})[\rm{diam}(\mathcal{A}) + \rm{diam}(\mathcal{S})])^2}\right).
      \end{multline}
      To simplify the analysis of the probability bound, by setting
      \begin{equation}
          t=\frac{L\mu(\mathcal{A}) \rm{diam}(\mathcal{A}) \delta}{\sqrt{K}}
      \end{equation}
      in \eqref{eq: nophi} for any $\delta > 0$, we obtain
      \begin{multline}\label{eq: nophi2} 
        \mathbb{P}\left(\Big{|}\frac{\mu(\mathcal{A})}{K} \sum_{k=1}^K \prod_{i=1}^N  F(o_t^{i},a_t^{i,k}) -\int_{\boldsymbol{a} \in \mathcal{A}}F(s_t,\boldsymbol{a})\mathrm{d} \boldsymbol{a} \Big{|} \ge \frac{\delta L\mu(\mathcal{A}) \rm{diam}(\mathcal{A}) }{\sqrt{K}} \right) \\
        \le 2\exp\left(-\frac{\delta^2}{2}\right)
    \end{multline}
      By using the triangle inequality and Assumption~\ref{assumption0}, we have
      \allowdisplaybreaks
      \begin{align}
          \Big{|}&\frac{\mu(\mathcal{A})}{K} \sum_{k=1}^K \prod_{i=1}^N  F(G_{\phi} (o_t^{i}, a_t^{i,k}),a_t^{i,k})  -\int_{\boldsymbol{a}:T(s,\boldsymbol{a})=s_t}F(s,a) \mathrm{d} \boldsymbol{a} \Big{|} \\ \le 
          \Big{|} &\frac{\mu(\mathcal{A})}{K}  \sum_{k=1}^K \prod_{i=1}^N F(G_{\phi^\star} (o_t^{i}, a_t^{i,k}),a_t^{i,k}) -\int_{\boldsymbol{a}:T(s,\boldsymbol{a})=s_t}F(s,a) \mathrm{d} \boldsymbol{a} \Big{|}   \\
          &+\Big{|}\frac{\mu(\mathcal{A})}{K}  \sum_{k=1}^K \prod_{i=1}^N F(G_{\phi^\star} (o_t^{i}, a_t^{i,k}),a_t^{i,k}) -\frac{\mu(\mathcal{A})}{K} \sum_{k=1}^K \prod_{i=1}^N F(G_{\phi} (o_t^{i}, a_t^{i,k}),a_t^{i,k}) \Big{|}  \\
          \le 
          \Big{|} &\frac{\mu(\mathcal{A})}{K}  \sum_{k=1}^K \prod_{i=1}^N F(G_{\phi^\star} (o_t^{i}, a_t^{i,k}),a_t^{i,k}) -\int_{\boldsymbol{a}:T(s,\boldsymbol{a})=s_t}F(s,a) \mathrm{d} \boldsymbol{a} \Big{|}  \\
          &+ \frac{\mu(\mathcal{A})}{K}  \sum_{k=1}^K \Big{|} \prod_{i=1}^N F(G_{\phi^\star} (o_t^{i}, a_t^{i,k}),a_t^{i,k}) - \prod_{i=1}^N F(G_{\phi} (o_t^{i}, a_t^{i,k}),a_t^{i,k}) \Big{|}  \\
          \le  \Big{|} &\frac{\mu(\mathcal{A})}{K}  \sum_{k=1}^K \prod_{i=1}^N F(G_{\phi^\star} (o_t^{i}, a_t^{i,k}),a_t^{i,k}) -\int_{\boldsymbol{a}:T(s,\boldsymbol{a})=s_t}F(s,a) \mathrm{d} \boldsymbol{a} \Big{|} +  \frac{\mu(\mathcal{A}) N^\alpha}{K^\beta}.
      \end{align}
      Then using \eqref{eq: phi_star} obtains 
      \begin{align} 
         &\mathbb{P}\left( \Big{|}\frac{\mu(\mathcal{A})}{K} \sum_{k=1}^K \prod_{i=1}^N F(G_{\phi} (o_t^{i}, a_t^{i,k}),a_t^{i,k}) -\int_{\boldsymbol{a}:T(s,\boldsymbol{a})=s_t}F(s,a) \mathrm{d} \boldsymbol{a} \Big{|} 
           \ge t \right) \\
           \le & \;
          \mathbb{P}\left( \Big{|}\frac{\mu(\mathcal{A})}{K} \sum_{k=1}^K \prod_{i=1}^N F(G_{\phi^\star} (o_t^{i}, a_t^{i,k}),a_t^{i,k}) -\int_{\boldsymbol{a}:T(s,\boldsymbol{a})=s_t}F(s,a) \mathrm{d} \boldsymbol{a} \Big{|} \ge t- \frac{\mu(\mathcal{A}) N^\alpha}{K^\beta} \right)
         \\
         \le & \; 2\exp\left(-\frac{K(t-\frac{\mu(\mathcal{A}) N^\alpha}{K^\beta})^2}{2(L \mu(\mathcal{A})[\rm{diam}(\mathcal{A}) + \rm{diam}(\mathcal{S})])^2}\right).\label{eq: sett}
      \end{align}
      To simplify the analysis of the probability bound, by setting
      \begin{equation}
          t=\frac{\delta L\mu(\mathcal{A}) [\rm{diam}(\mathcal{A}) + \rm{diam}(\mathcal{S})]}{\sqrt{K}} + \frac{\mu(\mathcal{A}) N^\alpha}{K^\beta}
      \end{equation}
      in \eqref{eq: sett} for any $\delta > 0$, we obtain
      \begin{multline}
        \mathbb{P}\left( \Big{|}\frac{\mu(\mathcal{A})}{K} \sum_{k=1}^K \prod_{i=1}^N F(G_{\phi} (o_t^{i}, a_t^{i,k}),a_t^{i,k})    - \int_{\boldsymbol{a}:T(s,\boldsymbol{a})=s_t}F(s,a) \mathrm{d} \boldsymbol{a} \Big{|} \right.  \\  \left.
          \ge \frac{\delta L\mu(\mathcal{A}) [\rm{diam}(\mathcal{A}) + \rm{diam}(\mathcal{S})]}{\sqrt{K}} + \frac{\mu(\mathcal{A}) N^\alpha}{K^\beta} \right)\\ 
        \le 2\exp\left(-\frac{\delta^2}{2}\right),
     \end{multline}
      Then we complete the proof.
\end{proof}

\end{document}